\documentclass[11pt,letterpaper]{article}
\usepackage[margin=1in]{geometry}
\usepackage{amsmath,amssymb}
\usepackage{hyperref}
\usepackage{caption}
\usepackage{subcaption}
\usepackage[noabbrev]{cleveref}
\usepackage{tikz}

\newcommand{\one}{\mathbf{1}}
\newcommand{\zero}{\mathbf{0}}
\newcommand{\eps}{\varepsilon}

\DeclareMathOperator*{\E}{\mathbb{E}}
\DeclareMathOperator*{\argmax}{arg\,max}

\DeclareMathOperator{\er}{er}

\DeclareMathOperator*{\spn}{span}
\DeclareMathOperator*{\mode}{mode}

\usepackage{amsthm,thmtools,thm-restate}
\newtheorem{theorem}{Theorem}[section]
\newtheorem{lemma}[theorem]{Lemma}

\theoremstyle{definition}
\newtheorem{definition}[theorem]{Definition}

\let\originalleft\left
\let\originalright\right
\renewcommand{\left}{\mathopen{}\mathclose\bgroup\originalleft}
\renewcommand{\right}{\aftergroup\egroup\originalright}

\makeatletter
\newcommand{\alphacmd@factory}[1]{}
\newcounter{alphacmdcounter}
\newcommand{\GenerateAlphabetCmds}[2]{%
    \renewcommand{\alphacmd@factory}[1]{%
        \expandafter\providecommand\csname #1##1\endcsname{{#2{##1}}}%
    }
    \setcounter{alphacmdcounter}{0}
    \loop
        \stepcounter{alphacmdcounter}
        \edef\alphacmd@ID{\@Alph\c@alphacmdcounter}
        \expandafter\alphacmd@factory\alphacmd@ID
    \ifnum\thealphacmdcounter<26
    \repeat
}
\newcommand{\GenerateAlphabetCmdsLower}[2]{%
    \renewcommand{\alphacmd@factory}[1]{%
        \expandafter\providecommand\csname #1##1\endcsname{{#2{##1}}}%
    }
    \setcounter{alphacmdcounter}{0}
    \loop
        \stepcounter{alphacmdcounter}
        \edef\alphacmd@ID{\@alph\c@alphacmdcounter}
        \expandafter\alphacmd@factory\alphacmd@ID
    \ifnum\thealphacmdcounter<26
    \repeat
}
\makeatother

\GenerateAlphabetCmds{}{\mathbb} 
\GenerateAlphabetCmds{c}{\mathcal} 
\usepackage[numbers,sort&compress]{natbib}

\usepackage{todonotes}

\title{The Sample Complexity of Replicable Realizable PAC Learning}
\author{
Kasper Green Larsen\thanks{Computer Science Department, Aarhus University, Email: \{larsen, markusm, clementks\}@cs.au.dk}
\and Markus Engelund Mathiasen\footnotemark[1]
\and Chirag Pabbaraju\thanks{Computer Science Department, Stanford University, Email: cpabbara@stanford.edu}
\and Clement Svendsen\footnotemark[1]
}
\date{}

\begin{document}
\maketitle
\begin{abstract}
    \noindent In this paper, we consider the problem of replicable realizable PAC learning. We construct a particularly hard learning problem and show a sample complexity lower bound with a close to $(\log|\cH|)^{3/2}$ dependence on the size of the hypothesis class $\cH$. Our proof uses several novel techniques and works by defining a particular Cayley graph associated with $\cH$ and analyzing a suitable random walk on this graph by examining the spectral properties of its adjacency matrix. 
    Furthermore, we show an almost matching upper bound for the lower bound instance, meaning \textit{if} a stronger lower bound exists, one would have to consider a different instance of the problem.
\end{abstract}

\section{Introduction}
Replicability in science is a notion of being able to replicate the work of fellow scientists with your own experiments and data. In recent years, replicability has become a more discussed topic as it has been pointed out in multiple Nature articles that we might face what is called a \emph{reproducibility crisis} \cite{baker2016reproducibility,ball2023ai}. When it comes to algorithms in machine learning, one might argue that these are perfectly reproducible, as long as the researchers share the source code along with the training data and the internal randomness used by the algorithm.
However, one might ask if it would be possible to design algorithms which give the same result even without sharing the training data. This would have the advantage of other researchers being able to verify that the training data was not cherry-picked, since they can run the algorithm on their own training data.
The notion of \emph{replicable} algorithms was introduced by \citet{reproducibility_in_learning} as a theoretical property of learning algorithms which captures this notion of being able to replicate the output of the algorithm even with new training data. More formally, they define a $\rho$-replicable learning algorithm as follows.
\begin{definition}[$\rho$-replicability~\cite{reproducibility_in_learning}]
    Let $\cA$ be a randomized algorithm. Then $\cA$ is $\rho$-replicable if there exists an $n\in \N$ such that for all distributions $\cD$ over some domain $\cX$, it holds that
    \[ \Pr_{S_1,S_2,r}\left[\cA(S_1; r) = \cA(S_2; r)\right] \geq 1-\rho \]
    where $S_1,S_2 \sim \cD^n$ are independent, and $r$ denotes the internal randomness used by $\cA$. 
\end{definition}
That is, as long as $\cA$ sees a sample from the same underlying distribution $\cD$, then it will with high probability output the same classifier.
We should remark that even though $\cA$ is run on different training data, we still require the same internal randomness to be used for both runs.
This turns out to be a necessary condition for many scenarios if we want such a guarantee, since otherwise, simple tasks such as mean estimation become impossible to do $\rho$-replicably if we don't share the internal randomness \cite{DixonListReplicability}.

In this work, we consider the setup of probably approximately correct learning (PAC learning) \cite{valiant1984theory}, which is the classic theoretical model for supervised learning. More specifically, we work with \emph{binary classification} in the \emph{realizable} setting. In this setup, we are interested in designing an algorithm which, with high probability, produces a classifier with good accuracy on new data.

More formally, a learning problem consists of a domain $\cX$, a label space $\{0, 1\}$, a hypothesis class $\cH \subseteq \{0, 1\}^\cX$, an unknown true hypothesis $h^\star \in \cH$, and an unknown distribution $\cD$ over $\cX$.
We denote the error of a classifier $h$ as $\er_\cD(h) = \Pr_{x\sim\cD}[h(x) \neq h^\star(x)]$.
Then, we say that a (randomized) algorithm $\cA$ is a PAC learner for $\cH$ if there exists a function $n: (0, 1)^2 \rightarrow \N$ such that for any $\eps,\delta \in (0, 1)$, $h^\star \in \cH$ and distribution $\cD$, if $\cA$ is given at least $n(\eps, \delta)$ i.i.d.\ samples from $\cD$ labeled according to $h^\star$, then with probability at least $1 - \delta$ over the randomness of the samples and the internal randomness of $\cA$, the classifier $g$ produced by $\cA$ will have $\er_{\cD}(g) \leq \eps$. The function $n$ is referred to as the \textit{sample complexity} of $\cA$. 

PAC learning for binary classification has been studied extensively \cite{ehrenfeucht1989general,auer1997learning,auer2007new,simon2015almost,steve,kasper,AdenAli2024majority,hogsgaard2025efficient}, and it is known that learnability is characterized by the VC dimension of $\cH$ and that the optimal sample complexity in the realizable setting is $\Theta\left(\frac{1}{\eps}(\text{VC}(\cH)+\log(1/\delta))\right)$ \cite{ehrenfeucht1989general,steve, kasper}. Replicable PAC learning differs from classical PAC learning in the sense that it is the Littlestone dimension \cite{littlestone1988learning} rather than the VC dimension that characterizes learnability. In particular,
in \cite[Cor. 2]{private_pac}, it is shown that a class that is \textit{privately} learnable has finite Littlestone dimension, whereas in \cite[Thm 3.1]{stability_is_stable}, it is shown that replicability implies privacy. Hence, a class that is replicably PAC learnable necessarily has finite Littlestone dimension. However, the sample complexity of replicable PAC learning is not fully understood. In the agnostic setting
\footnote{In agnostic PAC learning, the data is not assumed to be labeled by some $h^\star \in \cH$, but is instead drawn from a joint distribution $\cD$ over $\cX \times \{0,1\}$, and the error is measured relative to the hypothesis in $\cH$ that has the best error with respect to $\cD$.},
 there is an upper bound on the sample complexity for infinite $\cH$ which is \textit{exponential} in the Littlestone dimension \cite[Thm 1.4]{little} while for finite $\cH$, there is an upper bound which is \textit{quadratic} in $\log|\cH|$ with an almost matching lower bound \cite[Thm 5.13]{stability_is_stable}.

\subsection{Main Results}
The main result of this paper is a lower bound for replicable realizable PAC learning with finite hypothesis classes.
\begin{restatable}[Replicable Learning Lower Bound]{theorem}{lowerbound}\label{thm:lower_bound}
    For any integer $d \geq 10^{11}$, and positive reals $\eps,\delta,\rho \leq 10^{-4}$, there exists a domain $\cX$, a hypothesis class $\cH \subseteq\{0, 1\}^\cX$ with VC-dimension $d$, such that for any algorithm $\cA$ there is a distribution $\cD$ over $\cX$ for which $\cA$ needs at least
    \[
        n = \widetilde{\Omega}\left(\frac{(\log |\cH|)^{3/2}}{\eps}\right)
    \]
    labeled samples from $\cD$ in order to be a $\rho$-replicable PAC learner for $\cH$ with error $\eps$ and failure probability $\delta$. Here $\widetilde{\Omega}$ hides logarithmic factors in $\log|\cH|$ and $1/\eps$.
\end{restatable}
This is the first lower bound for replicable realizable PAC learning beyond the lower bounds for the non-replicable setup where one only needs $\Omega(\frac{1}{\eps}(\log|\cH| + \log(1/\delta))$. While our lower bound doesn't scale with $\rho$ and $\delta$, it does show a stronger dependence on $\log|\cH|$. As mentioned in the introduction, it is already known that one needs $\Omega(\log^2|\cH|)$ samples if we go to the agnostic setting. A natural question is therefore if the true dependence is in fact $(\log|\cH|)^2$. It turns out that for the instances we consider in the lower bound, this is not the case, since we can construct an algorithm with an almost matching upper bound for these instances.

\begin{restatable}[Replicable Learning Upper Bound]{theorem}{upperbound}\label{thm:upper_bound}
    There exists an algorithm $\cA$ such that for every instance $(\cX, \cH, \cD)$ shown to be hard in the proof of \Cref{thm:lower_bound}, and for every $\eps,\delta,\rho \in (0, 1)$, $\cA$ is a $\rho$-replicable PAC learner on this instance with sample complexity
    \[
        n = \widetilde{O}\left(\frac{(\log|\cH|)^{3/2}}{\rho\eps}\right).
    \]
\end{restatable}
This means that $(\log|\cH|)^{3/2}$ is not just an artifact of our proof, and if a stronger lower bound exists, one has to change our instances in some way to prove it.
However, it remains an open question whether such an instance exists, or if one can construct an upper bound which applies to all instances.

Beyond the above two theorems, we believe that the main contribution of this paper lies on the novel technical ideas used in both the upper and lower bound. We describe these in more detail in \Cref{sec:technical_contributions} along with an overview of the proofs. We believe the techniques presented here could prove useful in other contexts. 

\paragraph{Open Problems.}
Our results naturally raise the intriguing question of pin pointing the exact sample complexity of replicable realizable PAC learning with finite $\cH$. We make the careful conjecture that a $\rho^{-1}$ dependency on the replicability parameter is in fact possible for arbitrary $\cH$ and distribution $\cD$. Regarding the dependency on $\log |\cH|$, we are more divided. Our difficulties in extending our upper bound to other hypothesis sets $\cH$ might suggest a $(\log |\cH|)^2$ lower bound. On the other hand, our hard instance $(\cX, \cH, \cD)$ used in the lower bound is very similar to canonically hard instances in standard realizable PAC learning with no replicability requirements and thus could on the other hand suggest that $(\log |\cH|)^{3/2}$ is the right behavior. We hope our work may inspire further progress in understanding replicable learning.

\subsection{Further Related Work}
Since the paper by \citet{reproducibility_in_learning}, replicable algorithms have been designed for a wide variety of problems such as clustering \cite{clustering}, learning half spaces \cite{reproducibility_in_learning,kalavasis2024replicable}, online learning \cite{Online}, mean estimation \cite{replicability_highdims}, reinforcement learning \cite{reinforcement} and distribution testing \cite{diakonikolas2025replicable}. Furthermore, there have been interesting connections to other notions of stability such as global stability \cite{replicability_stability} and differential privacy \cite{stability_is_stable}. Various extensions and generalizations such as list-replicability \cite{replicability_stability} and approximate replicability \cite{hopkins2025approximate} have also been considered.

Lastly, a very recent paper from \citet{hopkins2025approximate} studies agnostic PAC learning under various weaker notions of \textit{approximate} replicability. Notably, for "pointwise" replicable algorithms, they show upper and lower bounds on the sample complexity which are linear in the VC dimension in the realizable setting. Hence, our lower bound shows a separation between fully replicable PAC learning and PAC learning with this weaker notion of replicability.

\citet{replicability_highdims} show a direct way to prove a lower bound on replicable mean estimation. Their approach considers a $d$-dimensional cube consisting of all combinations of $d$ biased coins where each coordinate corresponds to the bias of a coin. By picking the biases randomly, they show that the algorithm has to change output distribution appropriately in order to be replicable most of the time, meaning it will need a lot of samples. This approach is similar to our approach in the lower bound proof except that we consider a graph instead of a cube. Unfortunately, one cannot directly apply their result to get good lower bounds for realizable PAC learning, since they rely on the fact that changing the bias of a coin slightly doesn't change the output distribution of any algorithm that much. This is not necessarily the case for the instances we construct in \Cref{sec:technical_contributions}.

\subsection{Notation}
We briefly introduce some notation. Fix numbers $k,d \in \N$ where $k$ is prime, a space $\cX$, hypothesis class $ \cH\subseteq \{0,1\}^{\cX}$ and $\cD$ a distribution on $\cX$. Then we denote by $\Z_k$ the group of integers with addition $\bmod \,k$. We identify the elements of $\Z_k$ with their representatives in $\{0, \dots, k-1\}$, yielding a total ordering on $\Z_k$. We also consider the vector space $\Z_k^d$ over the field $\Z_k$ and equip it with the $\bmod\, k$ inner product denoted $\langle \cdot, \cdot \rangle$.
Furthermore, for hypotheses $h_1,h_2\in \cH$, let the error of $h_2$ with respect to $h_1$ be defined as
\[
\er_{h_1}(h_2) = \Pr_{x\sim \cD}[h_2(x)\ne h_1(x)].
\]
For a sample $S=(x_1,\dots,x_n) \in \cX^n$, let $h(S)=(h(x_1,)\dots,h(x_n))$. We denote by $\log$ the natural logarithm and $\log_b$ the base-$b$ logarithm. For a positive integer $m$, we define the set $[m]=\{0,\dots,m-1\}$.
\section{Technical Overview}\label{sec:technical_contributions}
In this section, we present the high-level ideas of our new lower bound for replicable realizable PAC learning, followed by the main ideas in a near-matching upper bound for the same data distribution and hypothesis set.

In both our lower and upper bounds, we consider the input domain $\cX = [d] \times \Z_k$ for a prime $k$. Our hypothesis set $\cH$ contains a hypothesis $h_{i}$ for every $d$-tuple $i = (i_0,\dots,i_{d-1}) \in \Z_k^d$. For a point $(a,b) \in \cX$, we have $h_i((a,b)) = 1$ if $i_a \leq b < i_a + \lfloor k/2\rfloor$ or if $b < i_a + \lfloor k/2\rfloor < i_a$. Otherwise $h_i((a,b))=0$. Each $h_i$ thus corresponds to outputting $1$ on $d$ intervals of length $\lfloor k/2\rfloor$ with wrap-around, where the $a^\text{th}$ interval starts at $i_a$. See \Cref{fig:hypothesis_example} for an illustration of these intervals. Note that $|\cH|=k^d$.
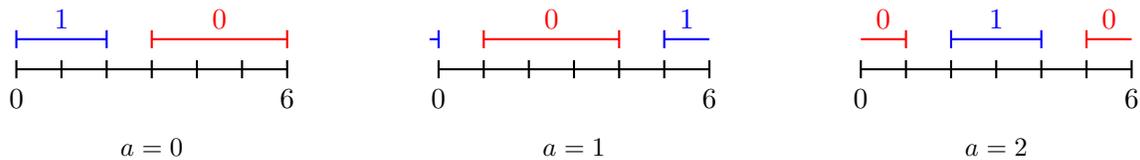
\begin{figure}[b]
    \captionsetup[subfigure]{labelformat=empty}
    \centering
    \newcommand\h{0.12}
    \newcommand\wf{0.6}
    \begin{subfigure}{0.32\textwidth}
        \centering
        \begin{tikzpicture}
            \draw[black, thick] (1*\wf,0) -- (7*\wf,0);
            \draw[black, thick] (1*\wf,\h) -- (1*\wf,-\h) node[below]{$0$};
            \draw[black, thick] (2*\wf,\h) -- (2*\wf,-\h);
            \draw[black, thick] (3*\wf,\h) -- (3*\wf,-\h);
            \draw[black, thick] (4*\wf,\h) -- (4*\wf,-\h);
            \draw[black, thick] (5*\wf,\h) -- (5*\wf,-\h);
            \draw[black, thick] (6*\wf,\h) -- (6*\wf,-\h);
            \draw[black, thick] (7*\wf,\h) -- (7*\wf,-\h) node[below]{$6$};
            \draw[blue, thick] (1*\wf, 0.4) -- (3*\wf, 0.4);
            \draw[blue, thick] (1*\wf, 0.4+\h) -- (1*\wf, 0.4-\h);
            \draw[blue, thick] (3*\wf, 0.4+\h) -- (3*\wf, 0.4-\h);
            \draw[blue, thick] (2*\wf, 0.4) node[above]{$1$};
            \draw[red, thick] (4*\wf, 0.4) -- (7*\wf, 0.4);
            \draw[red, thick] (4*\wf, 0.4+\h) -- (4*\wf, 0.4-\h);
            \draw[red, thick] (7*\wf, 0.4+\h) -- (7*\wf, 0.4-\h);
            \draw[red, thick] (5.5*\wf, 0.4) node[above]{$0$};
        \end{tikzpicture}
        \caption{$a = 0$}
    \end{subfigure}
    \hfill
    \begin{subfigure}{0.32\textwidth}
        \centering
        \begin{tikzpicture}
            \draw[black, thick] (1*\wf,0) -- (7*\wf,0);
            \draw[black, thick] (1*\wf,\h) -- (1*\wf,-\h) node[below]{$0$};
            \draw[black, thick] (2*\wf,\h) -- (2*\wf,-\h);
            \draw[black, thick] (3*\wf,\h) -- (3*\wf,-\h);
            \draw[black, thick] (4*\wf,\h) -- (4*\wf,-\h);
            \draw[black, thick] (5*\wf,\h) -- (5*\wf,-\h);
            \draw[black, thick] (6*\wf,\h) -- (6*\wf,-\h);
            \draw[black, thick] (7*\wf,\h) -- (7*\wf,-\h) node[below]{$6$};
            \draw[blue, thick] (6*\wf, 0.4) -- (7*\wf, 0.4);
            \draw[blue, thick] (6*\wf, 0.4+\h) -- (6*\wf, 0.4-\h);
            \draw[blue, thick] (1*\wf, 0.4+\h) -- (1*\wf, 0.4-\h);
            \draw[blue, thick] (0.8*\wf, 0.4) -- (1*\wf, 0.4);
            \draw[blue, thick] (6.5*\wf, 0.4) node[above]{$1$};
            \draw[red, thick] (2*\wf, 0.4) -- (5*\wf, 0.4);
            \draw[red, thick] (2*\wf, 0.4+\h) -- (2*\wf, 0.4-\h);
            \draw[red, thick] (5*\wf, 0.4+\h) -- (5*\wf, 0.4-\h);
            \draw[red, thick] (3.5*\wf, 0.4) node[above]{$0$};
        \end{tikzpicture}
        \caption{$a = 1$}
    \end{subfigure}
    \hfill
    \begin{subfigure}{0.32\textwidth}
        \centering
        \begin{tikzpicture}
            \draw[black, thick] (1*\wf,0) -- (7*\wf,0);
            \draw[black, thick] (1*\wf,\h) -- (1*\wf,-\h) node[below]{$0$};
            \draw[black, thick] (2*\wf,\h) -- (2*\wf,-\h);
            \draw[black, thick] (3*\wf,\h) -- (3*\wf,-\h);
            \draw[black, thick] (4*\wf,\h) -- (4*\wf,-\h);
            \draw[black, thick] (5*\wf,\h) -- (5*\wf,-\h);
            \draw[black, thick] (6*\wf,\h) -- (6*\wf,-\h);
            \draw[black, thick] (7*\wf,\h) -- (7*\wf,-\h) node[below]{$6$};
            \draw[blue, thick] (3*\wf, 0.4) -- (5*\wf, 0.4);
            \draw[blue, thick] (3*\wf, 0.4+\h) -- (3*\wf, 0.4-\h);
            \draw[blue, thick] (5*\wf, 0.4+\h) -- (5*\wf, 0.4-\h);
            \draw[blue, thick] (4*\wf, 0.4) node[above]{$1$};
            \draw[red, thick] (6*\wf, 0.4) -- (7*\wf, 0.4);
            \draw[red, thick] (6*\wf, 0.4+\h) -- (6*\wf, 0.4-\h);
            \draw[red, thick] (1*\wf, 0.4) -- (2*\wf, 0.4);
            \draw[red, thick] (2*\wf, 0.4+\h) -- (2*\wf, 0.4-\h);
            \draw[red, thick] (6.5*\wf, 0.4) node[above]{$0$};
            \draw[red, thick] (1.5*\wf, 0.4) node[above]{$0$};
        \end{tikzpicture}
        \caption{$a = 2$}
    \end{subfigure}
    \caption{Example with $k = 7$ for hypothesis $h_i$ for $i = (0, 5, 2)$. Each interval shows which values of $b$ will make $h_i((a, b)) = 1$. For instance, if $a = 1$ then $h_i((a, b)) = 1$ for $b \in \{0, 5, 6\}$.}\label{fig:hypothesis_example}
\end{figure}

Let $\cD$ be the uniform distribution on $\cX$. Let $h^\star=h_{i^\star} \in \cH$ be an unknown target function and assume training samples are drawn by sampling $S=x_1,\dots,x_n$ i.i.d.\ from $\cD$ and constructing the training set $(x_1,h^\star(x_1)),\dots,(x_n,h^\star(x_n))$.

\subsection{Lower Bound}

Let $\cA$ be a replicable PAC learning algorithm for the hypothesis set $\cH$. In our lower bound proof, we let $h^\star$ be drawn uniformly from $\cH$ and note that $h^\star$ is unknown to $\cA$, except through the labels of training samples. 

If $(S,h^\star(S))$ denotes a training set labeled by $h^\star$, then we use $\cA(S,h^\star(S);r)$ to denote the output of $\cA$ on the random string $r$ and training set $(S,h^\star(S))$.
Let $\eps$ be the error parameter of $\cA$, let $\delta$ be the failure probability, and let $\rho$ be the replicability parameter. That is, for any target function $h^\star \in \cH$, it holds with probability at least $1-\delta$ over the random choice of a training set $S \sim \cD^n$ and $r$ that 
\[
\er_{\cD}(\cA(S,h^\star(S);r)) := \Pr_{x\sim\cD}[\cA(S,h^\star(S) ; r)(x) \neq h^\star(x)] \leq \eps.
\]
Our goal is to lower bound $n$ as a function of $\eps$ and $|\cH| = k^d$, assuming $\delta,\rho$ are sufficiently small constants.

Similarly to previous lower bounds in replicable learning, see e.g.,~\cite{replicability_highdims,reproducibility_in_learning}, we start by fixing the internal randomness $r$ to obtain a deterministic algorithm. Observe that for a fixed $r^\star$, we can define the \emph{mode} of a hypothesis $h \in \cH$, denoted $\mode(h)$, as the most frequently reported hypothesis $\hat{h} \in \{0,1\}^\cX$ on a training set $(S,h(S))$ with $S \sim \cD^n$. That is,
\[
\mode(h) = \argmax_{\hat{h} \in \{0,1\}^{\cX}} \Pr_{S \sim \cD^n}[\cA(S,h(S);r^\star)= \hat{h}].
\]
For $\delta,\rho$ sufficiently small constants, we can show via Markov's inequality and a union bound that there is a fixed value $r^\star$ of the randomness, such that for $(99/100)|\cH|$ of the hypotheses $h$ in $\cH$, it must be the case that $\Pr_{S \sim \cD^n}[\cA(S,h(S);r^\star)=\mode(h)] \geq 99/100$, and at the same time, $\er_{h}(\mode(h)) \leq \eps$. Here $\er_h(\mode(h))$ denotes $\Pr_{x \sim \cD}[\mode(h)(x) \neq h(x)]$. We fix an arbitrary such $r^\star$ and let $\bar{\cH}$ be the set of at least $(99/100)|\cH|$ \emph{good} hypotheses satisfying these two properties. The good hypotheses $h$ thus often output their mode, and the mode has a high accuracy when data is labeled with $h$.

Our lower bound proof now consists of three main steps. In the first step, we show that if $n$ is small, then most pairs of hypotheses $h_1,h_2$ that label $\cX$ \emph{nearly identically} (in a sense to be made precise later) have the same mode. In the second step, we show that this implies that there are many hypotheses in $\bar{\cH}$ with the same mode. Finally, we argue that when many hypotheses in $\bar{\cH}$ have the same mode, then there must be a pair $h_1,h_2 \in \bar{\cH}$ with the same mode, but $h_1$ and $h_2$ are so different that it is not possible for $\mode(h_1)=\mode(h_2)$ to simultaneously satisfy $\er_{h_1}(\mode(h_1)) \leq \eps$ and $\er_{h_2}(\mode(h_1)) \leq \eps$. By definition of $\bar{\cH}$, this gives a contradiction (to $n$ being small). We now elaborate on the three steps.

\paragraph{Step 1.} Recall that all hypotheses in $\cH$ may be represented by a vector $(i_0,\dots,i_{d-1}) \in \Z_k^d$, each giving the starting point of an interval of length $\lfloor k/2 \rfloor$ in $\Z_k$. We can thus identify each hypothesis $h_u$ in $\cH$ with a vector $u \in \Z_k^d$. Let $Z = \{-1,0,1\}^d$, and consider the graph $G$ having one node for each hypothesis/vector $u \in \Z_k^d$ and an edge between nodes $u$ and $v$ if $u+z=v$ (mod $k$) for some $z \in Z$. Note that this is an undirected graph (since $e \in Z$ if and only if $-e \in Z$). The familiar reader may notice that $G$ is in fact the Cayley graph corresponding to the group $\Z_k^d$ with generator $Z$. We wish to show that if we sample $u$ uniformly in $\Z_k^d$ and let $v = u + z$ for uniform $z \in Z$, then the corresponding hypotheses $h_u, h_v$ are both in $\bar{\cH}$, \textit{and} have the same mode with probability at least $96/100$. 

To prove this, observe first that if a training set $S \sim \cD^n$ contains no samples that are labeled differently by two hypotheses $h_u$ and $h_v$, i.e.\ $h_u(S)=h_v(S)$, then since $\cA$ is deterministic (we fixed the randomness $r^\star$), we must have that $\cA(S,h_u(S);r^\star) = \cA(S,h_v(S);r^\star)$.
Now instead of letting $v = u+z$ for $z$ uniform in $Z$, consider first sampling $u$ uniformly, sampling $S \sim \cD^n$ and then picking $v$ among all hypotheses with $v-u \in Z$ satisfying $h_u(S)=h_v(S)$. Then we have that $\cA(S,h_u(S);r^\star) = \cA(S,h_v(S);r^\star)$. 
If we require $k \geq cn/d$ for large enough constant $c>0$, then notice that for each coordinate $a \in [d]$, if we let $S_a \subseteq S$ be the subset of samples in $S$ of the form $(a,b)$ for some $b \in \Z_k$, then with large constant probability (growing with $c$), we have $h_u(S_a)=h_{u + e_a}(S_a) = h_{u-e_a}(S_a)$ where $e_a$ is the $a^\text{th}$ standard unit vector and summation is over $\Z_k^d$. This follows simply from the fact that $S_a$ must contain one of the points $\{(a,u_a), (a,u_a-1), (a, u_a + \lfloor k/2\rfloor), (a, u_a + \lfloor k/2\rfloor -1)
\}$ for the label assignments to be distinct. By picking $v$ in a careful randomized way among all hypotheses with $h_u(S)=h_v(S)$ and $v-u \in Z$, we can now ensure that $v-u$ is precisely uniform in $Z$. This argument crucially needs that $h_u(S_a)=h_{u + e_a}(S_a) = h_{u-e_a}(S_a)$ for each coordinate with large constant probability. We thus obtain a distribution over triples $(u,v,S)$ so that $u$ is uniform in $\Z_k^d$, $v-u$ is uniform in $Z$, $S \sim \cD^n$ and $\cA(S,h_u(S);r^\star) = \cA(S,h_v(S);r^\star)$. 

Next, observe that $u$ is chosen independently of $S$. Thus with probability at least $98/100$, we have $\cA(S,h_u(S);r^\star)=\mode(h_u)$ and $h_u \in \bar{\cH}$. A careful argument also shows that if we consider the distribution of the pair $(S,v)$, then $v$ is uniform and independent of $S$. It is only through the variable $u$ that dependencies between $S$ and $v$ are introduced. Another union bound gives us that with probability at least $96/100$, we have that $h_u, h_v \in \bar{\cH}$ and $\mode(h_u) =\cA(S,h_u(S);r^\star) = \cA(S,h_v(S);r^\star) = \mode(h_v)$.
\paragraph{Step 2.} We now want to leverage the result of Step 1.\ to show that there are many pairs $h_u,h_v \in \bar{\cH}$ that are assigned the same mode. For this, consider partitioning the nodes $u$ of $G$ based on the modes $\mode(h_u)$. That is, for every $f \in \{0,1\}^{\cX}$ that appears as a mode, we let $C_f \subseteq \bar{\cH}$ denote the subset of nodes $u \in \bar{\cH}$ so that $\mode(h_u) = f$. We will show that there must be a large $C_f$. For this argument, notice that Step 1.\ implies that 
\begin{align*}
    96/100 &\leq \sum_f \Pr[u \in C_f \wedge u+z \in C_f] \\
    &= \sum_f \Pr[u+z \in C_f \mid u \in C_f]\Pr[u \in C_f].
\end{align*}
This implies the existence of an $f^\star$ so that $\Pr[u+z \in C_{f^\star} \mid u \in C_{f^\star}] \geq 96/100$. We claim that this property implies that $C_{f^\star}$ is large. To see this, notice that conditioning on $u \in C_{f^\star}$ simply means that $u$ is uniform in $C_{f^\star}$. Since $z$ is uniform in $Z$ and independent of $u$, this further implies that $(u,u+z)$ is uniformly random among all (directed) edges incident to the nodes in $C_{f^\star}$. We thus have that $C_{f^\star}$ is a set of nodes with \emph{small expansion} in the Cayley graph $G$. That is, at most a $4/100 = 1/25$ fraction of the directed edges $\{(u,v) : u \in C_{f^\star}, (u,v) \in E(G) \}$ has $v \notin C_{f^\star}$.

We thus proceed to show that every small set of nodes $T$ in $G$ expands a lot, i.e.,\ many of the incident edges have an end point not in $T$. For this, we consider the adjacency matrix $A$ of $G$ and let $\one_T \in \{0,1\}^{k^d}$ be an indicator vector for the nodes in $T$, taking the value $1$ in coordinates corresponding to nodes $u \in T$ and $0$ elsewhere. If only a $1/25$ fraction of the edges incident to nodes in $T$ leave $T$, then we must have $\langle A \one_T ,\one_T\rangle \geq (24/25)|T||Z|$. This follows since $\langle A \one_T ,\one_T\rangle$ counts precisely the number of directed edges $(u,v)$ with both $u,v \in T$. Furthermore, there is a total of $|T||Z|$ directed edges incident to $T$ since all nodes have degree $|Z|$. 

The eigenvectors and eigenvalues of the adjacency matrix of a Cayley graph are well understood. In particular, $A$ has an eigenvector corresponding to each vector $u \in \Z_k^d$. Let us denote this eigenvector by $\chi_u$ and the corresponding (real valued) eigenvalue by $\lambda_u$. We then have that
\begin{align*}
    \langle A \one_T ,\one_T\rangle &= \sum_{u\in \Z_k^d} \lambda_u |\langle \chi_u , \one_T \rangle|^2.
\end{align*}
One can show that all entries of $\chi_u$ are bounded by $k^{-d/2}$ in magnitude, yielding $|\langle \chi_u , \one_T \rangle|^2 \leq |T|^2 k^{-d}$. Furthermore, the eigenvectors form an orthonormal basis and thus $\sum_u |\langle \chi_u , \one_T \rangle|^2 = \|\one_T\|^2 = |T|$. Finally, since $G$ is a $|Z|$-regular graph, the largest eigenvalue of $A$ is $\lambda_\zero = |Z|$. To exploit these properties, let $\mu_1 \leq \mu_2 \leq \cdots \leq \mu_{k^d} = |Z|$ be the eigenvalues $\lambda_v$ in sorted order and $\chi_i$ the eigenvector corresponding to $\mu_i$. Combining our observations, we get that for any set $T$, we can upper bound $\langle A \one_T ,\one_T\rangle$ as
\begin{align*}
    \langle A \one_T ,\one_T\rangle &\leq \mu_{k^d} \cdot \sum_{i=k^d - k^d/(2|T|)+1}^{k^d}   |\langle \chi_i, \one_T \rangle|^2+ \mu_{k^d-k^d/(2|T|)} \cdot \sum_{i=1}^{k^d-k^d/(2|T|)} |\langle \chi_i, \one_T \rangle|^2\\
    &\leq |Z|\cdot \frac{k^d}{2|T|} \cdot \frac{|T|^2}{k^d} +  \mu_{k^d-k^d/(2|T|)} \cdot \left(|T| - \frac{k^d}{2|T|} \cdot \frac{|T|^2}{k^d} \right)\\
    &= \left(|Z|/2 + \mu_{k^d-k^d/(2|T|)}/2\right)|T|.
\end{align*}
Now if $T$ satisfies $\langle A \one_T ,\one_T\rangle \geq (24/25)|T||Z|$, we conclude that we must have $\mu_{k^d-k^d/(2|T|)} \geq (46/50)|Z|$. We thus proceed to bound the $k^d/(2|T|)$'th largest eigenvalue of the adjacency matrix $A$. Here we argue that for each $u \in \Z_k^d$, we have 
\[
\lambda_u = |Z| - 2 \sum_{z \in Z} \sin^2(\pi \langle u, z \rangle/k).
\]
To get a feel for this expression, consider a fixed non-zero $z \in Z$ and let $u$ be drawn uniformly from $\Z_k^d$. Then the inner product $\langle u , z \rangle$ is uniform in $\Z_k$. In particular, with probability close to $1/2$, the inner product lies in $\{\lfloor k/4\rfloor+1,\dots,\lfloor k/4\rfloor + \lfloor k/2\rfloor\}$, yielding $\sin^2(\pi \langle u ,z \rangle/k) \geq \sin^2(\pi/4) = 1/2$. So at least in expectation over a uniform $u$, we have $\E[\lambda_u] \leq |Z|-|Z|/2 \ll (46/50)|Z|$. To show that $|T|$ must be large, we thus need to argue that almost all $\lambda_u$ are close to this expectation. We do this via a probabilistic argument. In particular, we notice that for $\lambda_u$ to satisfy $\lambda_u \geq (46/50)|Z|$, there can be no more than $(8/50)|Z|$ values of $z$ for which $\langle u, z \rangle \in \{\lfloor k/4\rfloor + 1,\dots,\lfloor k/4\rfloor + \lfloor k/2\rfloor\}$. 

We now let $u$ be chosen uniformly from $\Z_k^d$ and define indicator random variables $X_z$ taking the value $1$ if $\langle u, z \rangle \notin \{\lfloor k/4\rfloor + 1,\dots,\lfloor k/4\rfloor + \lfloor k/2\rfloor\}$ and $0$ otherwise. If $p$ denotes the probability that $\sum_{z\in Z} X_z \geq (42/50)|Z|$, then for any $j > p k^d$, we have $\mu_{k^d-j} < (46/50)|Z|$. If we can give a good upper bound on $p$, we can now conclude that $k^d/(2|T|) \leq p k^d \Rightarrow |T| \geq p^{-1}/2$.

To bound the probability that $\sum_{z\in Z} X_z \geq (42/50)|Z|$, we consider the moment
\begin{align*}
\E\left[\left( \sum_{z \in Z} X_z - \E[X_z]\right)^r\right] &= \sum_{Y \in Z^r} \E\left[ \prod_{y_i \in Y} (X_{y_i} - \E[X_{y_i}]) \right],
\end{align*}
for an even $r \geq 2$. Here we notice that for any $Y \in Z^r$, if just one of the vectors $y_i \in Y$ is linearly independent of the remaining as vectors over $\Z_k^d$ (remember, $k$ is prime), then the variable $X_{y_i}$ is independent, and the whole monomial $\E\left[ \prod_{y_i \in Y} (X_{y_i} - \E[X_z]) \right]$ is zero. If $\beta$ denotes the fraction of tuples $Y \in Z^r$ where \emph{every} $y_i$ in $Y$ can be written as a linear combination of the remaining $y_j$, then we may bound
\begin{align*}
\E\left[\left( \sum_{z \in Z} X_z - \E[X_z]\right)^r\right] &\leq \beta |Z|^r.
\end{align*}
By Markov's inequality, we then have
\begin{align*}
    p=\Pr\left[\sum_{z\in Z} X_z \geq (42/50)|Z|\right] &\leq \Pr\left[\left|\sum_{z\in Z} (X_z-1/2)\right| \geq (17/50)|Z|\right] \\
    &= \Pr\left[\left(\sum_{z\in Z} (X_z-1/2)\right)^r \geq (17/50)^r|Z|^r\right] \\
    &\leq \frac{\E\left[\left(\sum_{z\in Z} (X_z-1/2)\right)^r\right]}{(17/50)^r |Z|^r} \\
    &\leq \beta (50/17)^r.
\end{align*}
That is, we get $|T| \geq p^{-1}/2 \geq \beta^{-1} (17/50)^r/2$. We thus seek a small upper bound on $\beta$. Here we again use a probabilistic argument. Consider drawing a set $Y \in Z^r$ uniformly at random, i.e.\ each $y_i$ in $Y$ is drawn independently and uniformly from $Z = \{-1,0,1\}^d$. Then $\beta$ is precisely the probability that every $y_i$ in $Y$ can be written as a linear combination of the remaining $y_j$. Giving a tight upper bound on this probability turns out to be rather involved. In particular, the fact that every $y_i$ can be written as a linear combination of the remaining does not imply that $\dim(\spn(Y))$ is small. It could be that for instance $\zero= \sum_i y_i$. The fact that the rank is at most $r-1$ only gives something like $\beta \leq c^{-d}$ for a constant $c \leq 3$. This turns out to be insufficient for our lower bound (as we shall see later). Instead, we first show that $\dim(\spn(Y)) \geq r-\log d$, except for a $d^{-d/2}$ fraction of $Y \in Z^r$. Assuming $\dim(\spn(Y)) \geq r-\log d$, we show that this implies that there must be a linear combination $y_i = \sum_{j \in J} \alpha_j y_j$ with $i \notin J$, $\alpha_j \neq 0$ and $|J| \geq r/\log d$. We then argue that such linear combinations involving many $y_j$'s are very unlikely when $k$ is large enough. In particular, the fact that $\zero = y_i -\sum_{j \in J} \alpha_j y_j$ implies that in each coordinate $a \in [d]$, we have that $0=y_i(a) - \sum_{j \in J} \alpha_j y_j(a)$. By definition of $Z$ and drawing $Y$ uniformly, we get that all the $y_i(a)$ and $y_j(a)$ are independent and uniform in $\{-1,0,1\}$. Using ideas by Golovnev et al.~\cite{golovnev2022polynomial} in a paper on data structure lower bounds in the group model, we can now use a Littlewood-Offord type anti-concentration result to show that $0=y_i(a) - \sum_{j \in J} \alpha_j y_j(a)$ only with probability roughly $k^{-1} + \sqrt{1/(|J|+1)}$. To get a feel for this claim, observe that it morally says that when $k$ is large enough, a sum $\sum_{j \in J} y_j \alpha_j$ with each $y_j$ uniform in $\{-1,0,1\}$ and each $\alpha_j$ non-zero, takes on any particular value $\bmod\ k$ with probability no larger than the probability that the random sum $\sum_{j \in J} y_j$ over the integers takes a particular value. The random sum $\sum_{j \in J} y_j$ has standard deviation $\Theta(\sqrt{|J|})$ and is near-uniform within one standard deviation of $0$, thus taking on any particular value with probability at most $O(1/\sqrt{|J|})$. Ignoring the dependence on $k$ for simplicity (recall we needed $k \geq cn/d$), we have $k^{-1} + \sqrt{1/(|J|+1)} \approx \sqrt{\log(d)/r}$. Using independence across the $d$ choices for $a$ finally bounds $\beta$ as roughly $(\log (d)/r)^{d/2}$. Picking $r \approx d$ gives $|T| \geq (cd/\log(d))^{d/2}$ for some constant $c>0$. Note that this is much stronger than the first bound on $\beta$ of $c^{-d}$ that one obtains simply from $\dim(\spn(Y)) \leq r-1$.

\paragraph{Step 3.} From Step 2.\ we have concluded that there is a set $C_{f^\star}$ with $|C_{f^\star}| \geq (cd/\log(d))^{d/2}$. Recall that $C_{f^\star}$ is defined as all nodes $u$ so that $h_u \in \bar{\cH}$ and $\mode(h_u) = f^\star$. These nodes $u$ thus correspond to hypotheses $h_u$ with the same mode, and where by definition of $\bar{\cH}$, this mode $f^\star$ has error at most $\eps$ when $\cX$ is labeled by $h_u$. Since $\cD$ is uniform over $\cX$ and $|\cX| = dk$, this implies that any two hypotheses $h_u,h_v$ with $u,v \in C_{f^\star}$ can disagree on the label of at most $2\eps dk$ points $(a,b) \in \cX$. When viewed as vectors $u,v \in \Z_k^d$, this basically corresponds to $\|u-v\|_1 \leq 2\eps k d$ (if we ignore wrap-around for simplicity). But an $\ell_1$ ball of radius $2 \eps k d$ has at most $\sum_{i=0}^{2 \eps k d} 2^d \binom{d + i -1}{d-1}$ points with integer coordinates inside it. Let us for simplicity assume $\eps k d > 2d$, then this number of points is roughly $(c' \eps k)^d$ for some constant $c'>0$. We therefore must have $(cd/\log(d))^{d/2} \leq |C_{f^\star}| \leq (c' \eps k)^d$. Recall again that we needed to set $k \geq c''n/d$ for a constant $c''>0$. Inserting $k = c'' n/d$ and taking $d$'th root finally gives
\begin{align*}
   k  = \Omega\left(\eps^{-1} \sqrt{d/\log(d)}\right) \Rightarrow n = \widetilde{\Omega}(\eps^{-1}d^{3/2}).
\end{align*}
If we instead state the lower bound as a function of $|\cH| = k^d$ and use $k=c''n/d$, we have $d = \log |\cH|/\log k = \widetilde{\Omega}(\log |\cH|)$ and the lower bound is $n = \widetilde{\Omega}(\eps^{-1} (\log |\cH|)^{3/2})$, where $\widetilde{\Omega}$ hides logarithmic factors in $\log |\cH|$ and $1/\eps$.

\subsection{Upper Bound}
We also present an upper bound for the same input domain $\cX = [d] \times \Z_k$, hypothesis set $\cH$ and distribution $\cD$  considered in the lower bound. Recall here that $\cH$ contains a hypothesis $h_i$ for every $d$-tuple $(i_0,\dots,i_{d-1}) \in \Z_k^d$ that for a point $(a,b) \in \cX$ returns $1$ if $i_a \leq b < i_a + \lfloor k/2 \rfloor$ or if $b < i_a + \lfloor k/2 \rfloor < i_a$. Each $h_i$ thus corresponds to outputting a $1$ on $d$ intervals of length $\lfloor k/2 \rfloor$ with wrap-around. The distribution $\cD$ is simply the uniform distribution on $\cX$. We let the unknown target function $h_{i^\star} \in \cH$ be arbitrary.

Our replicable algorithm $\cA$ is quite natural. From a sample $S \sim \cD^n$, we start by estimating each $i_a$. Since we expect to see $n/d$ samples of the form $(a,b)$ for each $a \in [d]$, we can estimate each $i^\star_a$ to within additive $\widetilde{O}(kd/n)$. Call these estimates $b_a^S$. Since $\cD$ is uniform, if we output the hypothesis $h_{b^S}$ with $b^S=(b_0^S,\dots,b_{d-1}^S)$, then $\er_{\cD}(h_{b^S})$ is essentially equal to $\|b^S - i^\star\|_1 d^{-1}$, except that the wrap-around may reduce the error further. To be slightly more formal, define for any $u,v \in \Z_k^d$ the metric $\nu : \Z_k^d \times \Z_k^d \to \{0,\dots,\lfloor k/2 \rfloor \}$ defined by $\nu(u,v) = \sum_{i=0}^{d-1} \min(u_i-v_i,v_i-u_i)$. Note here that we are working over the finite field $\Z_k$ and thus we are implicitly taking $\bmod\, k$ in $u_i - v_i$ and $v_i - u_i$. We can thus interpret $\nu(u,v)$ as the wrap-around $\ell_1$ distance between $u$ and $v$. For convenience, we will also use $\nu(u_i,v_i) = \min(u_i-v_i,v_i-u_i)$ to denote the one-dimensional version of $\nu$.

With this notation, we see that $\er_{\cD}(h_i) = 2 \nu(i,i^\star)/|\cX| = 2\nu(u,v)/(kd)$. It follows that if $n = \widetilde{\Omega}(d \eps^{-1})$ then $\nu(b^S, i^\star) = \widetilde{O}(kd^2/n) \leq \eps kd /4$ and thus $\er_\cD(h_{b^S}) \leq \eps /2$ as desired. Unfortunately it is not replicable to simply output $h_{b^S}$.

Instead we need to randomly round $b^S$ using shared randomness, such that for another i.i.d.\ sample $S' \sim \cD^n$, the rounding of $b^S$ and $b^{S'}$ are equal with probability at least $1-\rho$. Our goal is to round $b^S$ and $b^{S'}$ to a hypothesis $h_i$ with $\nu(i,b^S), \nu(i,b^{S'}) \leq \eps kd/4$. The triangle inequality would then give $\nu(i,i^\star) \leq \nu(i ,b^S) + \nu(b^S, i^\star) \leq \eps kd /4 + \eps kd /4 = \eps k d/2$, implying $\er_{\cD}(h_i) \leq \eps$.

For the randomized rounding, we use shared randomness to shuffle all hypotheses in $\cH$ uniformly at random. From sample $S$, $\cA$ now outputs the \emph{first} hypothesis $h_i$ in the shuffled order which satisfies $\nu(i,b^S) \leq \eps k d/4$. Correctness is thus guaranteed from the arguments above. The tricky part is to show that for two samples $S$ and $S'$, the first such hypothesis $h_i$ is the same with probability at least $1-\rho$.

For this analysis, let $h_i$ be the first hypothesis satisfying $\nu(i,b^S) \leq \eps k d/4$. We will argue that with probability at least $1-\rho/2$, we also have $\nu(i,b^{S'}) \leq \eps k d/4$. This implies that if $h_{i'}$ is the first hypothesis satisfying $\nu(i,b^{S'}) \leq \eps k d/4$, then $h_{i'}$ is no later than $h_i$ in the random ordering of $\cH$. A symmetric argument and a union bound shows that at the same time, $h_i$ is no later in the random ordering than $h_{i'}$ implying $h_i = h_{i'}$.

Now observe that if we fix an $S$ and $S'$ and condition on the event that $h_i$ is the first hypothesis in the shuffled $\cH$ with $\nu(i,b^S) \leq \eps k d/4$, then the distribution of $h_i$ is uniform random among all hypotheses with $\nu(i,b^S) \leq \eps k d/4$. To analyze $\nu(i,b^{S'})$, observe that in every coordinate $a \in [d]$, we have with probability $1/2$ (over $i$) that 
\[
\nu(i_a,b^{S'}_a) \leq \max\{\nu(i_a, b^S_a), \nu(b^S_a,b^{S'}_a)\} - \min\{ \nu(i_a, b^S_a), \nu(b^S_a,b^{S'}_a)\}.
\]
and with probability $1/2$, we have 
\[
\nu(i_a,b^{S'}_a) \leq \max\{\nu(i_a, b^S_a), \nu(b^S_a,b^{S'}_a)\} + \min\{ \nu(i_a, b^S_a), \nu(b^S_a,b^{S'}_a)\}.
\]
The two cases corresponds to whether $i_a$ is in the direction towards $b^{S'}_a$ from $b^S_a$ or in the opposite direction. If we have $n = \widetilde{O}(\beta^{-1} d)$ for a parameter $\beta>0$ to be fixed, then we can ensure that $\nu(b^S_a,b^{S'}_a) \leq \beta k/2$ for each coordinate $a$. Since we sample $h_i$ satisfying $\nu(i,b^S) \leq \eps k d/4$, we would expect most coordinates $a$ to have $\nu(i_a,b^S_a) \approx \eps k/4$. If we let $\beta \ll \eps$, then for most $a$, the $\min$ above is $\nu(b^S_a,b^{S'}_a)$. Let us for simplicity assume that this is always the minimum. Then there are signs $\sigma_a \in \{-1,1\}$ such that
\begin{align*}
    \nu(i,b^{S'}) &\leq \sum_{a=0}^{d-1} \max\{\nu(i_a, b^S_a), \nu(b^S_a,b^{S'}_a)\} + \sigma_a \min\{ \nu(i_a, b^S_a), \nu(b^S_a,b^{S'}_a)\} \\
    &=  \sum_{a=0}^{d-1} \nu(i_a, b^S_a)+ \sigma_a \nu(b^S_a,b^{S'}_a) \\
    &= \nu(i,b^S) + \sum_{a=0}^{d-1} \sigma_a \nu(b^S_a,b^{S'}_a)
\end{align*}
Noting that the signs $\sigma_a$ are uniform and independent, and that $\nu(b^S_a,b^{S'}_a) \leq \beta k/2$, we get from Hoeffding's inequality that the contribution from $\sum_{a=0}^{d-1} \sigma_a \nu(b^S_a,b^{S'}_a)$ is bounded by $\widetilde{O}(\beta k \sqrt{d})$ with high probability (like $1-\rho/4)$.

What remains is thus to show that $\nu(i,b^S)$ is somewhat smaller than $\eps k d/4$ with probability $1-\rho/4$. This amounts to a counting/probabilistic argument where we show that a uniform random $i$ satisfying $\nu(i,b^S) \leq \eps k d/4$ has $\nu(i,b^S) \leq \eps k d /4 - \Omega(\eps k \rho)$ with probability $1-\rho/4$. In some sense, this is showing that the $\ell_1$ "wrap-around" ball in $\Z_k^d$ has most points somewhat in the interior of the ball.

We now have that with probability $1-\rho/2$, it holds that $\nu(i,b^{S'}) \leq \eps k d/4 + \widetilde{O}(\beta k \sqrt{d}) - \Omega(\eps k \rho)$. Picking $\beta = \widetilde{O}(\eps \rho/\sqrt{d})$ gives $\nu(i,b^{S'}) \leq \eps k d/4$ and thus $h_i$ is also a valid output on $S'$ and we conclude $h_{i'}$ is no later than $h_i$ in the random order.

For the above, we needed $\nu(b^S_a,b^{S'}_a) \leq \beta k /2$. Since $\nu(b^S_a ,i_a) = \widetilde{O}(dk/n)$ we also have $\nu(b^S_a,b^{S'}_a) = \widetilde{O}(dk/n)$.
It is thus sufficient to pick an $n$ satisfying $n = \widetilde{O}(\beta^{-1} d) = \widetilde{O}(\eps^{-1} \rho^{-1} d^{3/2})$.
Since $|\cH| = k^d$, we have $d^{3/2} = O((\log |\cH|)^{3/2})$ and we have the claimed upper bound.

In our full proof, we also have to deal with the fact that not all coordinates $a$ have $\nu(i_a,b^S_a) \geq \nu(b^S_a,b^{S'}_a)$. This complicates the analysis somewhat, but the overall intuition and strategy remains the same. 

In summary, the two main observations are that, 1., the uniform $h_i$ among all hypotheses with $\nu(i,b^S) \leq \eps k d/4$ actually has $\nu(i,b^S)$ somewhat smaller than $\eps k d/4$ with good probability, and 2., for every coordinate $a$, there is a probability $1/2$ that $i_a$ is towards $b^{S'}_a$ from $b^S_a$ and thus $\nu(i_a,b^{S'}_a)$ is actually less than $\nu(i_a,b^S_a)$. This probability of $1/2$ is exactly what yields the $\sqrt{d}$ behavior (a sum of $d$ random signs has a standard deviation of $\sqrt{d}$).

\section{Proof of the Lower Bound}\label{sec:lower_bound_proof}
In this section, we will prove our sample complexity lower bound for replicable realizable PAC learning. For convenience, we restate the theorem here.
\lowerbound*
The instance we will prove is hard is the one described in \Cref{sec:technical_contributions}. So, let $\cX = [d]\times \Z_k$ for prime $k$ to be determined in a moment, and let $\cH$ contain the hypotheses $h_i$ for every $d$-tuple $i = (i_0, \dots, i_{d-1}) \in \Z_k^d$ given by 
\[
h_i((a, b)) = \begin{cases}
    1, \quad \text{if } i_a \leq b < i_a + \lfloor k/2\rfloor \text{ or } b < i_a + \lfloor k/2\rfloor < i_a\\
    0, \quad \text{otherwise}
\end{cases}
\]
Lastly, we will pick $\cD$ to be the uniform distribution over $\cX$.
We will pick $k$ to be a prime satisfying 
\[\max\left\{\frac{2}{\eps}, \frac{\sqrt{d}}{\log d}, \frac{4n}{\log(81/80)d}\right\}\leq k \leq 2\cdot \max\left\{\frac{2}{\eps}, \frac{\sqrt{d}}{\log d}, \frac{4n}{\log(81/80)d}\right\}.\] 
Remark that such a prime always exists due to the Bertrand-Chebyshev theorem.
To prove the lower bound, we assume that there exists a randomized algorithm $\cA$ which is a replicable PAC learner for $\cH$ with sample complexity $n$. The goal is then to show that $n$ must be large. By replicability and correctness of $\cA$, we have that for $S,S'\sim \cD^m$ and for any $h \in \cH$ it holds that
\begin{align}
    \Pr_{S,S',r}[\cA(S, h(S); r) = \cA(S', h(S'); r)] &\geq 1 - \rho,\label{ineq:replicable}\\
    \Pr_{S, r}[\er_h(\cA(S, h(S); r) \leq \eps] &\geq 1-\delta.\label{ineq:correct}
\end{align}
Specifically, (\ref{ineq:replicable}) and (\ref{ineq:correct}) hold if we pick $h$ uniformly at random in $\cH$.
Now, we would like to de-randomize $\cA$. Thus, we apply Markov's inequality and a union bound to get that with probability at least $1/3$ we pick a randomness $r^\star$ such that both
\begin{align}
    \Pr_{S,S',h}[\cA(S, h(S); r^\star) = \cA(S', h(S'); r^\star)] &\geq 1 - 3\rho,\label{ineq:det_replicable}\\
    \Pr_{S, h}[\er_h(\cA(S, h(S); r^\star) \leq \eps] &\geq 1-3\delta.\label{ineq:det_correct}
\end{align}
Therefore, consider a deterministic version of $\cA$, which uses a fixed randomness satisfying properties (\ref{ineq:det_replicable}) and (\ref{ineq:det_correct}). Remark that the sample complexity of the deterministic $\cA$ will lower bound the sample complexity of the randomized $\cA$. For the rest of the proof, we will therefore consider this deterministic version of $\cA$ and lower bound its sample complexity. For ease of notation, we will not write the fixed randomness used by $\cA$ explicitly in the rest of the proof.
After de-randomizing $\cA$, we can now define the mode of a hypothesis with respect to $\cA$.
\begin{definition}[Mode]
    Let $h\in \cH$. Then we define the \emph{mode} of $h$ as
    \[
        \mode(h) = \argmax_{f\in\{0,1\}^\cX}\left\{\Pr_{S\sim\cD^m}[\cA(S, h(S)) = f]\right\}.
    \]
\end{definition}
In words, the mode of $h$ is just the most likely output of $\cA$ when the samples are labeled by $h$.
To make the mode uniquely determined, we break ties by choosing the lexicographically smallest $f$ in the $\argmax$.

Now, we arrange all the hypotheses in a graph $G = (V, E)$. Here, the vertices are the vectors associated to the hypotheses in the graph. That is $V = \Z_k^d$, so we have one node for each hypothesis. Now, let $Z = \{-1, 0, 1\}^d$. Then, for each $u,v \in \Z_k^d$, $E$ contains the edge $(u,v)$ if and only if $u - v \in Z$. We can therefore think of $G$ as a $d$-dimensional grid with diagonal edges and wrap-around. This can also be described as the Cayley graph on the group $\Z_k^d$ with generating set $Z$.

Now, for each possible output $f$ of $\cA$ we create a subset of nodes $u$ with $\mode(h_u) = f$ for which the error $f$ with respect to $h_u$ is less than $\eps$. Formally, for each $f \in \{0, 1\}^\cX$ we define such subset as
\[
    C_f = \{u \in \Z_k^d \mid \mode(h_u) = f,\ \er_{h_u}(f) \leq \eps\}.
\]
We then have the following lemma about these subsets.

\begin{restatable}{lemma}{lowexpansionsubset}\label{lem:low_expansion_subset}
    There exists a function $f^\star \in \{0, 1\}^\cX$ such that
    \[
        \#\{(u,v) \in E \mid u,v \in C_{f^\star}\} \geq \frac{24}{25}|C_{f^\star}||Z|.
    \]
\end{restatable}
Remark that $|C_f||Z|$ is the total number of edges incident to nodes in $C_f$. Consequently, \Cref{lem:low_expansion_subset} implies that for at least one $f^{\star}$, the set $C_{f^{\star}}$ has low expansion. We defer the proof of this lemma to \Cref{sec:random_step}. For now, assume such a subset exists and consider the following theorem.
\begin{restatable}{theorem}{expansion}\label{thm:expansion}
    Let $T\subseteq \Z_k^d$ be a subset of the vertices. If $\#\{(u,v) \in E \mid u,v\in T\}\ge (24/25)|T||Z|$ then $|T|\ge \left(\frac{\sqrt{d}}{450\log d }\right)^d$.
\end{restatable}
We will defer the proof of this theorem to \Cref{sec:expansion}.
We can now combine this with \Cref{lem:low_expansion_subset} to lower bound the size of $C_{f^\star}$ as $|C_{f^\star}| \geq \left(\frac{\sqrt{d}}{450\log d}\right)^d$. Now, to upper bound the size of $C_{f^\star}$, we have the following lemma.
\begin{restatable}{lemma}{cfupper}\label{lem:Cf_upper}
    For any $f \in \{0, 1\}^\cX$ we have $|C_f| \leq\ (6\eps k)^d$.
\end{restatable}
The proof of this lemma is deferred to \Cref{sec:l1_ball_bound}. We can then combine the upper and lower bound on $|C_{f^\star}|$ to get
\begin{equation}\label{eq:Cf_bounds}
    \left(\frac{\sqrt{d}}{450\log d}\right)^d \leq |C_{f^\star}| \leq (6\eps k)^d
\end{equation}
Remember that we chose $k$ such that $k \leq \max\left\{\frac{4}{\eps}, \frac{2\sqrt{d}}{\log d}, \frac{8n}{\log(81/80)d}\right\}$. We will now see that the first two terms in the max cannot be the maximum. Indeed, if the first term was the maximum, then inequality (\ref{eq:Cf_bounds}) would imply that
\[
    \frac{\sqrt{d}}{\log d} < 10800
\]
which contradicts the assumption that $d \geq 10^{11}$. If the second term was the maximum, then inequality (\ref{eq:Cf_bounds}) imply that
\[
    \frac{1}{5400} \leq \eps
\]
which contradicts the fact that $\eps \leq 10^{-4}$. Thus, it must be the case that $k \leq \frac{8n}{\log(81/80)d}$. Substituting this value of $k$ into inequality (\ref{eq:Cf_bounds}) gives us
\[
    n \geq 10^{-7} \frac{d^{3/2}}{\eps\log d}.
\]
We will now write the lower bound in terms of $|\cH|$. Using that $d = \log|\cH|/\log k \geq \log|H| / (7\log n)$, and $d \leq \log|\cH|$, we have
\[
    n \geq 10^{-9} \frac{(\log |\cH|)^{3/2}}{\eps \log\log|\cH|\log n}.
\]
We still have a dependence of $n$ on the right side of the inequality. However, a standard trick using a small proof by contradiction (see \Cref{lem:remove_recursion}), one can show that the above implies
\[
    n \geq 10^{-9}\frac{(\log|\cH|)^{3/2}}{\eps (\log\log|\cH|)^2\log(\log(|\cH|)/\eps)}.
\]
This is the claim of the lower bound, since we only need to remove logarithmic factors in $\log|\cH|$ and $1/\eps$. What remains is therefore to show \Cref{lem:low_expansion_subset}, \Cref{thm:expansion}, and \Cref{lem:Cf_upper}. Each of these will be proved in the following sections.

\subsection{Random Step Approach}\label{sec:random_step}
In this section, we will prove \Cref{lem:low_expansion_subset}. For this, we will need the following lemma about random steps in the graph $G$.
\begin{lemma}[Random step]\label{lem:random_step}
    Let $h_{u}$ be a uniformly random hypothesis in $\cH$, and let $h_{v}$ be a hypothesis chosen by starting at $h_{u}$ and taking a uniformly random step in the graph $G$ (that is, $v = u + z$ for $z$ uniform on $Z$). Then, we have
    \[
        \Pr_{h_u,h_v}[\mode(h_u) = \mode(h_v)] \geq 1 - 6\rho.
    \]
    Furthermore, with probability $1-6(\delta+\rho)$ over the choice of $h_u,h_v$, it holds that
    \begin{align*}
        \er_{h_u}(\mode(h_u)) &\leq \eps,\\
        \er_{h_v}(\mode(h_v)) &\leq \eps.
    \end{align*}
\end{lemma}
First, we will show how \Cref{lem:random_step} implies \Cref{lem:low_expansion_subset}. Then we will prove \Cref{lem:random_step}. For convenience, we also restate \Cref{lem:low_expansion_subset} here.
\lowexpansionsubset*
\begin{proof}[Proof of \Cref{lem:low_expansion_subset}]
    Consider $h_u, h_v$ chosen as in \Cref{lem:random_step}. Using that $\rho,\delta \leq \frac{1}{450}$, we have from \Cref{lem:random_step} with probability at least $24/25$ that $\mode(h_u) = \mode(h_v)$, $\er_{h_u}(\mode(h_u)) \leq \eps$, and $\er_{h_v}(\mode(h_v)) \leq \eps$. Now, these three statements together imply that there exists an $f \in \{0, 1\}^\cX$ such that $u, v \in C_{f}$. Also, define $f^* = \argmax_f \Pr[v \in C_f \mid u \in C_f]$. Since all $C_f$ are disjoint, we can write this out as
    \begin{equation}\label{eq:stay_prob}
        \frac{24}{25} \leq \sum_f\Pr[v \in C_{f}\land u \in C_{f}]
        = \sum_f \Pr[v \in C_f \mid u \in C_f]\Pr[u \in C_f]
        \leq \Pr[v \in C_{f^\star} \mid u \in C_{f^\star}]
    \end{equation}
    Remark that since $u$ is chosen uniform in $\Z_k^d$ independently of everything else, conditioning on $u \in C_{f^\star}$ just means that $u$ is uniform in $C_{f^\star}$. Now, since $z$ is uniform in $Z$, this implies that $(u,v)$ is a uniformly random among all edges incident to nodes in $C_{f^\star}$. Since there are $|C_{f^\star}||Z|$ of such edges, we have
    \[
        \Pr[v \in C_{f^\star} \mid u \in C_{f^\star}] = \frac{\#\{(u, v) \in E \mid u,v \in C_{f^\star}\}}{|C_{f^\star}||Z|}
    \]
    Combining this with (\ref{eq:stay_prob}) gives us the desired result.
\end{proof}
What remains in this section is to prove \Cref{lem:random_step}. However, to do so, we will need two additional lemmas. The first one states that the mode is output with high probability. The second one tells us that there is a way to pick a uniformly random neighbor in $G$ while being indistinguishable with respect to samples $S$.

\begin{lemma}[Mode is likely]\label{lem:mode_likely}
    Let $S \sim \cD^n$ be $n$ i.i.d.\ samples from $\cD$, and let $h \sim \cH$ be a uniformly random hypothesis in $\cH$ independent of $S$. Then,
    \[
        \Pr_{S,h}[\cA(S, h(S)) = \mode(h)] \geq 1-3\rho.
    \]
\end{lemma}
\begin{proof}
    Consider another sample $S'\sim\cD^n$, independent of both $S$ and $h$. We can then use the replicability property of $\cA$ (see inequality (\ref{ineq:det_replicable})) to get that
    \begin{align*}
        1-3\rho &\leq \E_{h}\left[\Pr_{S,S'}[\cA(S, h(S) = \cA(S', h(S')]\right]
        = \E_{h}\left[\sum_{f \in \cH}\Pr_{S}[\cA(S, h(S) = f]^2\right] \\
        &\le \E_{h}\left[\Pr_S[\cA(S,h(S))=\mode(h)]\sum_{f \in \cH}\Pr_{S}[\cA(S, h(S) = f]\right] \\
        &= \E_{h}\left[\Pr_S[\cA(S,h(S))=\mode(h)]\right] 
        = \Pr_{S,h}[\cA(S, h(S)) = \mode(h)].\qedhere
    \end{align*}
\end{proof}
\begin{lemma}\label{lem:down_sampling}
Let $u$ be a uniformly random node in the graph $G$ and let $S \sim \cD^n$ be independent of $u$. Then, there exists a way to pick $v$ a neighbor of $u$ such that $h_u(S) = h_v(S)$ and $v$ is a uniformly random neighbor of $u$. Furthermore, $v$ and $S$ are independent.
\end{lemma}
Let us make some remarks on \Cref{lem:down_sampling}. In particular the dependencies among the random variables are delicate. If we consider the whole triple of random variables $(u,u-v,S)$ then they are \emph{not} mutually independent. This should be clear from the fact that the requirement $h_u(S)=h_v(S)$ disallows some choices of $v$ given $u$ and $S$. Instead, the random variables $(u, u-v, S)$ are \emph{pair-wise independent}. So any two of the variables behave as a uniform random and independently chosen pair. This pair-wise independence is critical for our proof.

We defer the proof of \Cref{lem:down_sampling} to later in this section. For now, we show how these lemmas imply \Cref{lem:random_step}.
\begin{proof}[Proof of \Cref{lem:random_step}]
    The proof will be divided into two parts. Part 1 will show the first inequality of the lemma, and part 2 will show the second and third inequality.
    \paragraph{Part 1.}
    Let $S \sim \cD^n$ be independent of $u$. Then, we invoke \Cref{lem:down_sampling} to take a uniformly random step from $u$ to a node $v$ while making sure $h_u(S) = h_v(S)$.
    Now, since $S$ is independent of $u$ and $v$ (although they are not all three independent of each other), and $h_u, h_v$ are both uniformly random hypotheses in $\cH$, we can use \Cref{lem:mode_likely} on both $u$ and $v$, to get that
    \begin{align*}
        \Pr_{S,h_u,h_v}[\cA(S, h_u(S)) = \mode(h_u)]=\Pr_{S,h_u}[\cA(S, h_u(S)) = \mode(h_u)] \geq 1-3\rho,\\
        \Pr_{S,h_u,h_v}[\cA(S, h_v(S)) = \mode(h_v)]=\Pr[\cA(S, h_v(S)) = \mode(h_v)] \geq 1-3\rho.
    \end{align*}
    Now, using the fact that $h_u(S) = h_v(S)$, a union bound tells us that
    \[
        \Pr_{h_u,h_v}[\mode(h_u) = \mode(h_v)] \geq 1 - 6\rho.
    \]
    \paragraph{Part 2.}For the second part, we will need the correctness property (\ref{ineq:det_correct}) of $\cA$ and the fact that $h_u$ is a uniformly random hypothesis in $\cH$. First, let $h \sim \cH$ be a uniformly random hypothesis in $\cH$, and let $S \sim \cD^n$ be $n$ i.i.d.\ samples drawn from $\cD$, independently of $h$.
    The correctness property of $\cA$ requires that
    \begin{equation}\label{ineq:correctness}
        \Pr_{h, S}[\er_{h}(\cA(S,h(S))) \leq \eps] \geq 1-3\delta.
    \end{equation}
    In words, this just says that the error of the classifier produced by $\cA$ can be more than $\eps$ with probability at most $3\delta$. We can then use the law of total probability to get that
    \begin{align*}
        1 - 3\delta &\leq \Pr_{h,S}[\er_h(\cA(S, h(S)) \leq \eps, \cA(S, h(S)) = \mode(h)]
        \\&\quad+ \Pr_{h,S}[\er_h(\cA(S,h(S))) \leq \eps, \cA(S, h(S)) \neq \mode(h)]\\
        &\leq \Pr_h[\er_h(\mode(h)) \leq \eps] + \Pr_{h,S}[\cA(S, h(S)) \neq \mode(h)] \\
        &\leq \Pr_h[\er_h(\mode(h)) \leq \eps] + 3\rho,
    \end{align*}
    where the last inequality follows from \Cref{lem:mode_likely}. This implies that
    \[
        \Pr_h[\er_h(\mode(h)) \leq \eps] \geq 1 - 3(\delta + \rho).
    \]
    Now, since both $h_u$ and $h_v$ are uniformly random hypotheses in $\cH$, the second part of the lemma follows by a union bound over the above inequality instantiated for these two.
\end{proof}
Finally, to prove \Cref{lem:down_sampling} we will need one last lemma.
\newcommand{\majorization}{
    
}
\begin{restatable}{lemma}{majorization}\label{lem:majorization}
    Given values $x_0, \dots, x_d$ and $y_0, \dots, y_d$ with the following properties:
    \begin{itemize}
        \item (Non-negative). $x_0, \dots, x_d \geq 0$ and $y_0, \dots, y_d \geq 0$,
        \item (Equal sum). $\sum_{i=0}^d x_i = \sum_{j=0}^d y_j$,
        \item (Dominating). For all $k \in \{0, \dots, d\}$, we have $\sum_{i=0}^k x_i \leq \sum_{j=0}^k y_j$.
    \end{itemize}
    Then there exist values $p_{i,j}$ where $0 \leq j \leq i \leq d$, such that
    \begin{enumerate}
        \item for all $0 \leq j \leq i \leq d$, we have $0 \le p_{i,j} \leq 1$, \label{prop:non_neg}
        \item for all $i \in \{0, \dots, d\}$, we have $\sum_{j=0}^i p_{i,j} = 1$,\label{prop:sum}
        \item for all $j \in \{0, \dots, d\}$, we have $\sum_{i=j}^d x_i \cdot p_{i,j} = y_j$.\label{prop:eq_q}
    \end{enumerate}
\end{restatable}
We will note that an almost identical version of this lemma is proved by \citet[Lemma 1]{majorization_proof}; \citet[Theorem 3.4]{note_on_majorization}.
The only difference is that they require $x$ and $y$ to be ordered. This means they can compute the values $p_{i,j}$ more efficiently. However, in this lower bound, we only need the existence of such values, and we cannot make sure that $x,y$ are ordered. For completeness, we will therefore include a proof of this version in \Cref{sec:appendix}.
We are now ready to prove \Cref{lem:down_sampling}.
\begin{proof}[Proof of \Cref{lem:down_sampling}]
    We will describe a different way of taking a step in the graph and then show that this has the same distribution as taking a uniformly random step.
    
    First, let $\sigma_0, \dots, \sigma_{d-1} \in \{-1, 1\}$ be $d$ uniformly random signs, and let 
    \begin{equation}
        \label{eq:P-def}
        P = \{i \in [d] \mid h_{u}(S) = h_{u + \sigma_i e_i }(S)\},
    \end{equation}
    where $e_i$ is the $i^\text{th}$ standard basis vector. That is, $P$ is the set of all axis-aligned directions in which we could take a step from $h_u$ to $h_{u + \sigma_i e_i}$ without seeing any change in labels over the sample $S$. The way we will take a step is then to pick a subset $P' \subseteq P$ and compute $z' = \sum_{i \in P'}\sigma_i e_i$. We then take a step in direction $z'$. We therefore prove that $z'$ is uniform on $Z$ if we pick the subset $P'$ in the right way. Let $z$ be a uniformly random element of $Z$. Notice since the directions $\sigma_i$ are uniformly random, it is sufficient to show that $P'$ has the same distribution as $Q = \{i \in [d] \mid z_i \neq 0\}$. We will first argue that we can make $|P'|$ have the same distribution as $|Q|$. 
    
    To do this, we can invoke \cref{lem:majorization} with $x_i = \Pr[|P| = i]$ and $y_i = \Pr[|Q| = i]$. Then $x:=x_0,\dots,x_d$ and $y := y_0,\dots,y_d$ clearly satisfy the non-negativity and equal sum requirements of the lemma; what remains to be shown is that $\Pr[|P| \leq t] \leq \Pr[|Q| \leq t]$ for all $t \in \{0, \dots, d\}$. If this is indeed the case, then the lemma statement implies that there exists a distribution which tells us how many elements to remove from $P$ to make $|P'|$ have the same distribution as $|Q|$. Namely, if one samples $P$ with $|P|=i$, then one can obtain $P'$ by removing $i-j$ elements from $P$ with probability $p_{i,j}$, for the $p_{i,j}$ values given by the lemma.
    We can then choose the elements to remove from $P$ uniformly at random. Because of symmetry in the contents of $P$ and $Q$, this would make sure that the elements that \textit{remain} in $P'$ are uniformly random, and thus $P'$ will have the same distribution as $Q$. We will now show that $\Pr[|P| \leq t] \leq \Pr[|Q| \leq t]$ for every $t\in \{0, \dots, d\}$.

    First, note that $|Q| \sim \text{Binomial}\left(d, \frac{2}{3}\right)$. We therefore have
    \begin{equation}\label{eq:Q_lower}
        \Pr[|Q| \leq t] = \sum_{i=0}^t\binom{d}{i}\left(\frac{2}{3}\right)^i\left(\frac{1}{3}\right)^{d-i} \geq \sum_{i=0}^t\binom{d}{i}3^{-d}.
    \end{equation}
    Now before looking at $|P|$, we introduce some notation. For any $h_u$, and for a specific direction $i$, $h_{u}(x) \neq h_{u+\sigma_i e_i}(x)$ for exactly two points $x \in \{x_i^1, x_i^2\}$. This follows directly from the definition of $h_u$, since we can only distinguish $h_u$ and $h_{u+\sigma_i e_i \text{ mod } k}$ in the endpoints of the interval it induces. Therefore, the event $P = \{l_1, \dots, l_i\}$ implies
    \[
        \bigcap_{j \notin P}\{x_j^1 \in S \lor x_j^2 \in S\}.
    \]
    In words, this just means that $S$ contained either endpoint of the interval for all directions $j \notin P$. Notice that the events in the intersection above (across different values of $j$) are \textit{negatively correlated}; namely, given that one of  $x_j^1$ or $x_j^2$ is in $S$, it is less likely for $x_{j'}^1$ or $x_{j'}^2$ to also be in $S$. Furthermore, by symmetry, $\Pr[P = \{l_1, \dots, l_i\}] = \Pr[P=\{1, \dots, i\}]$. Using these observations, we can bound $\Pr[|P| \leq t]$ as:
    \begin{align*}
        \Pr[|P| \leq t] &= \sum_{i=0}^t \Pr[|P| = i]
        = \sum_{i=0}^t \binom{d}{i}\Pr[P = \{1, \dots, i\}]\\
        &\leq \sum_{i=0}^t \binom{d}{i}\Pr[\cap_{j=i+1}^d\{x_j^1 \in S \lor x_j^2 \in S\}]\\
        &\leq \sum_{i=0}^t \binom{d}{i}\prod_{j=i+1}^d\Pr[x_j^1 \in S \lor x_j^2 \in S]\tag{negative correlation}\\
        &=\sum_{i=0}^t\binom{d}{i}\prod_{j=i+1}^d\left(1 - \Pr[x_j^1 \notin S \land x_j^2 \notin S]\right)\\
        &= \sum_{i=0}^t\binom{d}{i}\prod_{j=i+1}^d\left(1-\left(1-\frac{2}{kd}\right)^n\right)\\
        &= \sum_{i=0}^t\binom{d}{i}\left(1-\left(1-\frac{2}{kd}\right)^n\right)^{d-i}\\
        &\leq \sum_{i=0}^t\binom{d}{i}\left(1 - \exp\left(\frac{-4n}{kd}\right)\right)^{d-i} \tag{$1-x \ge e^{-2x}$ for $x \in [0,3/4]$}\\
        &\leq \sum_{i=0}^t\binom{d}{i} \left(\frac{1}{81}\right)^{d-i}. \tag{since $k \geq \frac{4n}{\log(81/80)d}$}
    \end{align*}

    \medbreak
    Now note that the $i$'th term in the sum above is smaller than the $i$'th term in (\ref{eq:Q_lower}) as long as $t \leq 3d/4$. For $d > t > 3d/4$, we can instead bound the complementary event.
    \begin{align*}
        \Pr[|P| \leq t] &= 1 - \Pr[|P| > t] \leq 1 - \Pr[|P| = d]
        = 1 - \left(1 - \frac{2d}{kd}\right)^n\\
        & \leq 1 - \exp(-4n/k)
        \leq 1 - \left(\frac{80}{81}\right)^{d/10} \leq 1 - \left(0.997\right)^d.
    \end{align*}
    For $|Q|$, we can do a Chernoff bound, using $t > 3d/4$, to get that
    \begin{align*}
        \Pr[|Q| > t] \leq \Pr[|Q| > 3d/4] \leq \exp\left(\frac{-d}{288}\right) \leq \left(0.997\right)^d
    \end{align*}
    and thus, for $t > 3d/4$, we have
    \begin{align*}
        \Pr[|P| \leq t] \leq 1 - \left(0.997\right)^d \leq 1 - \Pr[|Q| > t] = \Pr[|Q| \leq t].
    \end{align*}
    This finishes the down-sampling part of the proof. 
    
    We now move on to proving independence of $v$ and $S$.
    We have already shown for any fixed $a \in \Z_k^d$ that $\Pr[v = a] = k^{-d}$ without fixing $S$. To show independence of $v$ and $S$, it is therefore enough to argue that $\Pr[v = a \mid S] = k^{-d}$. Therefore, fix the sample $S$. Then, suppose for any $b \in \Z_k^d$, the following equality holds:
    \begin{equation}\label{eq:swap_ab}
    \Pr[v = a \mid S, u = b] = \Pr[v = b \mid S, u = a].
    \end{equation}
    Then we can see that
    \begin{align*}
        \Pr[v = a \mid S] &= \sum_{b \in \Z_k^d}\Pr[v = a \mid S, u = b]\Pr[u = b \mid S]\\
        &= k^{-d}\sum_{b \in \Z_k^d} \Pr[v = a \mid S, u = b] \tag{$u$ and $S$ independent}\\
        &= k^{-d}\sum_{b\in\Z_k^d}\Pr[v = b \mid S, u = a]\\
        &= k^{-d}.
    \end{align*}
    We therefore just need to prove equation (\ref{eq:swap_ab}).

    Note that by definition of the sampling of $v$, both $\Pr[v=a \mid S, u=b]$ and $\Pr[v=b \mid S, u=a]$ are $0$ when $h_a(S) \neq h_b(S)$. Hence, we will restrict our attention to values of $a$ and $b$ that satisfy $h_a(S) = h_b(S)$. In this case,  let us further condition on the value of $\sigma \in \{-1,1\}^d$ and recall that $\sigma$ is drawn uniformly at random, and independently of both $u$ and $S$. Additionally, for any $\sigma$, 
    define $\sigma'$ as
    \begin{equation}
    	\label{def:sigma_transform}
    	\sigma' = \left(\sigma_i\cdot (-1)^{\one\{a_i \neq b_i\}}\right)_{i=1}^d.
    \end{equation}
	That is, $\sigma'$ flips the sign of $\sigma_i$ for all $i$ where $a_i \neq b_i$. We will argue that  
	\begin{equation}
		\label{eq:condition_sigma}
		\Pr[v = a \mid S, \sigma, u = b]=\Pr[v=b \mid S, \sigma', u=a].
	\end{equation}
	We can then conclude that
	\begingroup
	\allowdisplaybreaks
	\begin{align*}
		\Pr[v = a \mid S, u = b] &= \sum_{\sigma \in \{-1, 1\}^d} \Pr[v = a \mid S,\sigma, u = b]\Pr[\sigma|S, u=b]\\
		&= 2^{-d}\sum_{\sigma \in \{-1, 1\}^d}\Pr[v = a \mid S, \sigma, u = b] \tag{$\sigma$ independent of both $u$ and $S$}\\
		&= 2^{-d}\sum_{\sigma \in \{-1, 1\}^d}\Pr[v = b \mid S, \sigma', u = a] \tag{equation \ref{eq:condition_sigma}} \\
		&= 2^{-d}\sum_{\sigma' \in \{-1, 1\}^d}\Pr[v = b \mid S, \sigma', u = a] \\
		&= \sum_{\sigma' \in \{-1, 1\}^d} \Pr[v = b \mid S,\sigma', u = a]\Pr[\sigma'|S, u=a] \\
		&= \Pr[v = b \mid S, u = a],
	\end{align*}
	\endgroup
	which is the desired equality.
    To show equation \ref{eq:condition_sigma}, notice first that by definition of the sampling of $v$, $\Pr[v = a \mid S, \sigma, u = b]=0$ if $\sigma$ does not satisfy:
    \begin{equation}
        \label{eq:condition-a}
        \exists D\subseteq  [d]: a = \left(b + \sum_{i \in D}\sigma_ie_i\right) \text{ mod } k. 
    \end{equation}
	Note how this implies, by virtue of how $\sigma'$ is defined, that if $\sigma$ does not satisfy equation \ref{eq:condition-a}, then it is also the case that $\Pr[v=b \mid S, \sigma', u=a]=0=\Pr[v = a \mid S, \sigma, u = b]$.
	
 	Now, let us consider a value of $\sigma$ that satisfies equation \ref{eq:condition-a}. If we condition on $\sigma$, along with $S$ and $u=b$, this completely determines the random variable $P$ (defined in equation \ref{eq:P-def}) to be some set $P_1 \subseteq [d]$ (and in fact, the set $D$ which witnesses equation \ref{eq:condition-a} is some subset of $P_1$). Furthermore, in this case, $\Pr[v = a \mid S, \sigma, u = b]=p_{i,j} \cdot \binom{i}{\|a-b\|_1}^{-1}$, where $i=|P_1|$ and $j=\|a-b\|_1$, for the $p_{i,j}$ values given by \cref{lem:majorization}.
    
    We then claim that, when we condition on $S$, but with $u=a$ and the signs being $\sigma'$ instead, the value $P_2$ that the random variable $P$ gets determined to be is \textit{still equal to} $P_1$. 
    To see this, first consider any $i \in [d] \setminus P_1$. For such an $i$, it holds that $a_i=b_i$, and hence $\sigma_i=\sigma'_i$. Since $i$ was not included in $P_1$, $h_b(S) \neq h_{b+\sigma_i e_i \text{ mod } k}(S)$, which also means that $h_{a}(S) \neq h_{a + \sigma'_i e_i \text{ mod }k}(S)$. So, $i \neq P_2$. 
    
    Now, consider any $i \in P_1 \setminus D$. Again, for any such $i$, $a_i=b_i$ and hence $\sigma_i=\sigma'_i$. Consider any point $(x_1,x_2) \in S$ with $x_1=i$; since $i \in P_1$, it holds that $h_b((x_1,x_2))=h_{b+\sigma_i e_i \text{ mod } k}((x_1,x_2))$, and hence, $h_a((x_1,x_2))=h_{a+\sigma'_i e_i \text{ mod } k}((x_1,x_2))$. Since these are the only points affected, we have that $h_a(S)=h_{a+\sigma'_i e_i \text{ mod } k}(S)$, meaning that $i \in P_2$.
    
    Finally, consider any $i \in D$. We have that $a_i = b_i+\sigma_i \text{ mod } k$, meaning that $\sigma'_i = -\sigma_i$. Consider any point $(x_1,x_2) \in S$ with $x_1=i$; since $i \in P_1$, it holds that $h_b((x_1,x_2))=h_{b+\sigma_i e_i \text{ mod } k}((x_1,x_2))$. But this immediately also  means that $h_{a+\sigma'_i e_i \text{ mod } k}((x_1,x_2))=h_{a}((x_1,x_2))$. Since these are the only points affected, we have that $h_a(S)=h_{a+\sigma'_i e_i \text{ mod } k}(S)$, meaning that $i \in P_2$.

    In summary, we have argued that $\Pr[v=b \mid S, \sigma', u=a]$ is also equal to $p_{i,j} \cdot \binom{i}{\|a-b\|_1}^{-1}$, where $i=|P_2|=|P_1|$ and $j=\|b-a\|_1$; but this is also the value of $\Pr[v = a \mid S, \sigma, u = b]$, concluding the proof.
\end{proof}

\subsection{Expansion Property of \texorpdfstring{$G$}{G}}\label{sec:expansion}
In this section, we will prove \Cref{thm:expansion}. For convenience, we restate the theorem here.
\expansion*
Throughout this section, we use the notation $R(Y)$ for the rank of any matrix $Y$ and remind the reader that $\langle \cdot,\cdot\rangle$ denotes the $\bmod\, k$ inner product on $\Z_k^d$ considered as a vector space over the field $\Z_k$. Also, let $A$ denote the adjacency matrix of $G$.
The proof is split into several parts. 
\subsubsection{Relating Expansion Property to the Spectrum of \texorpdfstring{$A$}{A}}
Let $\one_T \in \{0,1\}^{k^d}$ denote the indicator vector on $T$. Then one easily checks that
\begin{align*}
    \#\{(u,v) : u,v\in T\} = \langle A \one_T,\one_T\rangle.
\end{align*}
In order to better understand this number, we compute the eigenvectors and eigenvalues of $A$. The spectral theory of Cayley graphs on Abelian groups is well understood \cite{trevisannotes,nica2018brief}, 
but we provide all proofs here for completeness. 
\begin{lemma}[Eigenvectors and eigenvalues of $A$]\label{lem: eigen}
     For any $v\in \Z_k^d$, let $\chi_v$\footnote{While the letter $\chi$ is usually used for the characters of the group, we emphasize here that our $\chi_v$ are the characters scaled by $k^{-d/2}$ to ensure unit norm. } be the vector indexed by $\Z_k^d$ with entries $\chi_v(w) = k^{-d/2} \exp(2 \pi i \langle v ,w \rangle/k)$. Then $\chi_v$ is an eigenvector of $A$ with corresponding eigenvalue 
     \[
     \lambda_v = |Z| - 2  \sum_{z \in Z} \sin^2(\pi \langle v,z \rangle/k).
     \]
     Furthermore, the collection $(\chi_v)_{v\in \Z_k^d}$ is an orthonormal basis of $\C^{k^d}$.
\end{lemma}
\begin{proof}
    Let $Z^+$ denote the set of vectors in $Z$ where the first non-zero coordinate is $1$.  observe that 
\begin{align*}
  (A \chi_v)(w) &= \sum_{z \in Z} \chi_v(w + z) \\
              &= \chi_v(w + (0)^d) + \sum_{z \in Z^+} \left(\chi_v(w + z)  + \chi_v(w-z) \right) \\
              &= \chi_v(w) + k^{d/2} \chi_v(w) \sum_{z \in Z^+}\left(\chi_v(z)  + \chi_v(-z) \right) \\
              &= \chi_v(w) + \chi_v(w) \sum_{z \in Z^+} \left(\exp(2 \pi i \langle v, z \rangle/k)  + \exp(-2 \pi i \langle v, z \rangle/k) \right)\\
              &= \chi_v(w) + \chi_v(w) \sum_{z \in Z^+} 2 \cos(2 \pi \langle v,z \rangle/k)\\
              &=\chi_v(w) + \chi_v(w) \sum_{z \in Z^+} (2-4\sin^2(\pi \langle v,z \rangle/k))\\
              &=|Z| \chi_v(w) - 2\chi_v(w)  \sum_{z \in Z} \sin^2(\pi \langle v,z \rangle/k). 
\end{align*}
This shows that $\chi_v$ is indeed an eigenvector with the claimed eigenvalue. We now show that they are an orthonormal basis
\begin{align*}
    \langle \chi_v,\chi_w\rangle_{\C} & = \sum_{u\in \Z_k^d}\chi_v(u)\overline{\chi_w(u)} = \sum_{u\in \Z_k^d}\chi_v(u)\chi_{-w}(u) = \sum_{u\in \Z_k^d}\chi_{v-w}(u).
\end{align*}
If $v=w$, each term in the sum is equal to $k^{-d}$, making the whole expression $1$. If instead $v\ne w$, there is at least one coordinate $j$ such that $v_j-w_j\ne 0$. Hence
\begin{align*}
    k^{-d}\sum_{u\in \Z_k^d}\chi_{v-w}(u) &= k^{-d}\sum_{u_1\in \Z_k}e^{2\pi i(v_1-w_1)u_1/k}\cdots \sum_{u_d\in \Z_k}e^{2\pi i(v_d-w_d)u_d/k}.
\end{align*}
Looking at just the $j$'th sum:
\[
\sum_{u_j\in \Z_k}e^{2\pi i(v_j-w_j)u_j/k} = \sum_{\ell=0}^{k-1}\left(e^{2\pi i(v_j-w_j)/k}\right)^{\ell} = \frac{\left(e^{2\pi i (v_j-w_j)/k}\right)^{k}-1}{\left(e^{2\pi i(v_j-w_j)/k}\right)-1} = 0.
\]
Hence, in this case $\langle \chi_v,\chi_w\rangle_\C=0$. Thus, the eigenvectors have unit length and are orthogonal, so they comprise an orthonormal basis 
\end{proof}

We now show how to bound $\langle A \one_T,\one_T\rangle$ in terms of the eigenvalues. First, let $\mu_1,\dots, \mu_{k^d}$ be the eigenvalues of $A$ sorted in increasing order and $\chi_i$ be the eigenvector associated with $\mu_i$.
\begin{lemma}\label{lem: upper_bound_edges}
    It holds that
    \[
    \langle A\one_T,\one_T\rangle \le |T|\left(|Z|/2+\mu_{k^d-k^d/(2|T|)} \right)
    \]
\end{lemma}
\begin{proof}
    By expanding $\one_T$ in the eigenvector basis and using the fact that $\chi_v$ is an eigenvector of $A$, we can then write
\[
\langle A \one_T,\one_T\rangle = \sum_{v\in \Z_k^d}|\langle \one_T,\chi_v\rangle|^2\lambda_v.
\]
From the definition of $\chi_v$, $ |\langle \one_T,\chi_v \rangle|^2\le |T|^2 k^{-d}$. Hence
\[
\sum_{i=k^d-k^{d}/(2|T|)+1}^{k^d} |\langle \one_T,\chi_i \rangle|^2 \le |T|/2. 
\]
Also, it follows from Lemma \ref{lem: eigen} that $\lambda_v\le |Z|$ for all $v$ and also that $\lambda_0=|Z|$. Hence, the largest eigenvalue is $\mu_{k^d} = |Z|$. This implies that
\[
\sum_{i=k^d-k^{d}/(2|T|)+1}^{k^d} \mu_i |\langle \one_T,\chi_i \rangle|^2 \le 3^{d} |T|/2.
\]
Furthermore, $\|\one_T\|^2=|T|$ meaning that
\[
\sum_{i=1}^{k^d-k^{d}/(2|T|)} \mu_i |\langle \one_T,\chi_i\rangle|^2 \le \mu_{k^d-k^{d}/(2|T|)}|T|.
\]
Combining all this yields
\[
\langle A\one_T,\one_T\rangle \le |T|\left(|Z|/2+\mu_{k^d-k^d/(2|T|)} \right).
\]
\end{proof} 

Using lemma \ref{lem: upper_bound_edges}, we can bound the size of $|T|$ in terms of the tail probabilities of a suitable binomial random variable related to the eigenvalues. For this, let $\cI=\{\lfloor{k/4}\rfloor+1,\dots,\lfloor{k/4}\rfloor+\lfloor{k/2}\rfloor\}$ and note that $|\cI|=\lfloor{k/2}\rfloor$ and let $u$ be a uniform random variable on $\Z_k^d$. For each $z \in Z$, define an indicator $X_z$ taking the value $1$ if $\langle u,z\rangle \notin \cI$ and $0$ otherwise.
\begin{lemma}\label{lem: size of |T|}
     Define 
    \[
    p = \Pr\left[\sum_{z\in Z}X_z \ge (42/50)|Z|\right].
    \]
    Then $|T|\ge p^{-1}/2$.
\end{lemma}
\begin{proof}
Consider a $u\in \Z_k^d$ such that 
\[
\sum_{z\in Z} \one\{\langle u,z\rangle \notin \cI \} < (42/50) |Z|
\]
Then,
\begin{align*}
    \lambda_u & =|Z|-2  \sum_{z \in Z} \sin^2(\pi \langle u,z \rangle/k) \\
    & \le |Z|-2  \sum_{z \in Z} \one\{\langle u,z\rangle \in  \cI \}\sin^2(\pi \langle u,z \rangle/k) \\
    & < |Z|-2\sum_{z\in Z} \one\{\langle u,z\rangle \in \cI\} \sin^2(\pi/4)\\
    & \le |Z|-(8/50)|Z| = (42/50)|Z| \le (46/50)|Z|.
\end{align*}
Negating the implication we just showed, we then have that
\[
\lambda_u > (46/50)|Z| \Rightarrow \sum_{z\in Z} \one\{\langle u,z\rangle \notin \cI \} \ge (42/50)|Z|).
\]
This then means that over uniformly random $u\in \Z_k^d$
\begin{align*}
    \Pr[\lambda_u > (46/50)|Z|] \le \Pr\left[\sum_{z\in Z}X_z \ge (42/50)|Z|\right] = p.
\end{align*}
Hence, for any $j > p k^d$, we have $\mu_{k^d-j} < (46/50)|Z|$. In particular, if $1/(2|T|)\ge p$ it must be the case that $\mu_{k^d-k^d/(2|T|)}<(46/50)|Z|$.
However, recalling Lemma \ref{lem: upper_bound_edges} and our assumptions, we have
    \[
    (24/25)|T||Z|\le \langle A\one_T,\one_T\rangle \le |T|\left(|Z|/2+\mu_{k^d-k^d/(2|T|)}/2 \right),
    \]
    such that $\mu_{k^d-k^d/(2|T|)} \geq (46/50)|Z|$. Thus, it must be the case that $\frac{1}{2|T|}< p$ or equivalently
\[
|T|>p^{-1}/2. \qedhere
\]
\end{proof}

\subsubsection{Bounding \texorpdfstring{$p$}{p}}
The strategy is now to bound $p$ using Markov's inequality with sufficient control over the central moments. Specifically, for $r\le d$ we bound 
\begin{align*}
  \E\left[\left|\sum_{z \in Z} (X_z-\E[X_z])\right|^r\right] &= \\
  \sum_{Y \in Z^r} \E\left[\prod_{i=0}^{r-1} (X_{y_i}-\E[X_{y_i}])\right],
\end{align*}
where we identify $Z^r$ with the set of matrices $\{-1,0,1\}^{d\times r}$.
Since $k$ is the power of a prime, it holds that $X_{y_i}$ is independent of $(y_j)_{j\ne i}$ if $y_i$ is linearly independent of $(y_j)_{j\ne i}$ over $\mathbb{F}_k^d$ (see Lemma \ref{lem: independence}). Hence, 
\[
\sum_{Y \in Z^r} \E\left[\prod_{i=0}^{r-1} (X_{y_i}-\E[X_{y_i}])\right] = \sum_{\substack{Y \in Z^r \\ y_i \in \spn(y_j : i\ne j)\forall i } }\E\left[\prod_{i=0}^{r-1} (X_{y_i}-\E[X_{y_i}])\right].
\]
We must then bound the number of $Y=(y_0,\dots,y_{r-1})\in Z^r$ that satisfy $y_i \in \spn(y_j : i\ne j)$ for all $i\in [r]$. We split this into two parts depending on the rank of the matrix $Y$. We begin with a simple lemma.
\begin{lemma}\label{lem: size of bla}
 Let $V$ be an $r$-dimensional subspace of $\F_k^d$. Then $|V \cap \{-1,0,1\}^d| \leq 3^r$.
\end{lemma}
\begin{proof}
Let $v_0,\dots,v_{r-1}$ be an arbitrary basis of $V$ and consider the matrix $A$ with $v_i$'s as rows. Since the column rank is equal to the row rank, there is a set $R$ of $r$ linearly independent columns. If $x \in \F_k^r$ is a  vector so that $xA \in \{-1,0,1\}^d$, then in particular, $xA_R \in \{-1,0,1\}^r$. There are $3^r$ choices for $xA_R$, and any such choice forces $r$ linearly independent constraints on $x$, each resulting in a unique choice of $x$. It follows that $|V \cap \{-1,0,1\}^d| \leq |\{x \in \F_k^r : xA_R \in \{-1,0,1\}^r\}| \leq 3^r$.
\end{proof}
With this lemma in hand, we can bound on the number of $Y\in Z^r$ with small rank. 
\begin{lemma}\label{lem: small rank}
   Assume $r\le d/2$. Then the number of $Y\in Z^r$ with rank $R(Y)\le r-\log_3(d)$ is at most $d^{-d/2}/d|Z|^r$.
\end{lemma}
\begin{proof}
    Consider drawing $Y$ one vector $y_i$ at a time, each obtained by sampling each coordinate independently, taking the values $\{-1,0,1\}$ uniformly. Now for all $\ell\in [r]$, let $W_\ell$ be a subspace of $\F_k^d$ of dimension $\ell-1$ having maximal intersection with $Z$ (which is less than $3^{\ell-1}\le 3^{r-1}$ by Lemma \ref{lem: size of bla}). Then letting $s=\log_3(d)$,
    \begin{align*}
        \Pr[\dim\spn(y_0,\dots,y_{r-1})\le r-s] &= \Pr[\exists i_0,\dots,i_{s-1} \in [r] : y_{i_\ell}\in \spn\{y_{i_0},\dots,y_{i_\ell-1}\} \,\forall \ell\in [s]] \\
        &\le \Pr[\sum_{\ell=0}^{r-1} \one\{y_\ell \in W_\ell\}\ge s] = \sum_{j=s}^r\Pr[\sum_{\ell=0}^{r-1} \one\{y_\ell \in W_\ell\}=j] \\
        & \le \sum_{j=s}^r \binom{r}{j}\left(3^{(r-1-d)}\right)^j \le 2^r 3^{s(r-d-1)} =2^r d^{r-d-1} \\
        & \le 2^{d/2}d^{-d/2}/d = \left(\frac{d}{2} \right)^{-d/2}/d
    \end{align*}
     Multiplying with $|Z|^r$ to get the total number then yields the result. \qedhere
\end{proof}

We now move on to bound the number of large rank matrices with the additional property that each column lies in the span of the other columns. 

\begin{lemma}\label{lem: lemma for big rank}
    Let $y_0,\dots,y_{r-1}$ be vectors in $\F_k^d$ for prime $k$ with the property that every $y_i\in \spn(y_j : i \ne j)$ and $\dim(\spn(Y)) = r-s$ with $s \geq 1$. Then there is a subset of indices $S \subseteq [r]$ with $|S| \geq r/s$ so that $\dim(\spn(\{y_i\}_{i \in S})) = |S|-1$ but for any subset of $S' \subseteq S$ with $|S'|=|S|-1$, the corresponding vectors are linearly independent. 
\end{lemma}
\begin{proof}
Assume for the sake of contradiction that the claim is false. Let $M$ be the matrix having the $y_i$'s as rows.

Now consider the following process for constructing a basis $b_0,\dots,b_{s-1}$ for the left nullspace of $M$, one vector at a time. Initialize $T=\emptyset$ as the set of indices $j$ so that at least one $b_i$ among already constructed basis vectors has non-zero $j$'th coordinate. For $i=0,\dots,r-1$, pick a vector $y_j$ with $j \notin T$ and write it as a linear combination $y_j = \sum_{h \neq j} \alpha_h y_h$. If there are multiple such linear combinations, pick one with the smallest number of non-zero coefficients $\alpha_h$, breaking ties arbitrarily. Then the vector $b_i$ with $h$'th coordinate $\alpha_h$ for $h \neq j$ and $j$'th coordinate $-1$ is in the left nullspace of $M$. Furthermore, it is linearly independent of $b_0,\dots,b_{i-1}$ since these vectors all have $0$ in their $j$'th coordinate by definition of $T$. We thus add it as the $i$'th basis vector and update $T \gets T \cup \{h : \alpha_h \neq 0\} \cup \{j\}$.

We now bound $|S|$ with $S=\{h : \alpha_h \neq 0\} \cup \{j\}$. We now show that the vectors $\{y_h \}_{\alpha_h \neq 0}$ are linearly independent. Indeed, if there was a linear dependence, we could write $\sum_{h \neq j} \alpha_h y_h$ as a linear combination of a smaller number of vectors, contradicting the fact that we picked a linear combination with the smallest number of non-zero coefficients. Thus $\dim(\spn(\{y_i\}_{i \in S})) = |S|-1$ and the subset $S' = S \setminus \{j\}$ is linearly independent. Finally, consider any subset $S'$ with $|S'|=|S|-1$ excluding any other $h \neq j$. Again, if there was a linear dependency, we could again write $y_j$ as a linear combination of fewer vectors, contradicting the choice of coefficients $\alpha_h$. Hence, since we assumed the theorem is not true, it must be that case that $|S|<r/s$, meaning that each $b_i$ adds fewer than $r/s$ indices to $T$. Thus upon selecting $b_i$, we have $|T| < (i-1)r/s$. The process can thus continue until $i=s+1$. This contradicts that the nullspace has dimension $s$.
\end{proof}

Before proceeding, we need a variant of the classical Littlewood-Offord lemma.

\begin{lemma}
  \label{lem:little}
  Let $\F_k$ be a prime field. Let $s \geq 1$, $x_0,\dots,x_{s-1} \in \F_k \setminus \{0\}$ and $y \in \F_k$. Then for $\eps_i$ sampled independently and uniformly in $\{-1,0,1\}$ we have
  \[
    \Pr\left[\sum_{i=0}^{s-1} \eps_i x_i = y\right] \leq \min\left\{\frac{1}{2}, \frac{1}{k} +\exp(-s/8) + \sqrt{\frac{32}{s} } \right\}.
  \]
\end{lemma}
\begin{proof}
  We use Lemma 2.5 in \citet{golovnev2022polynomial}
  to conclude that the probability is at most
  \begin{align*}
    &\sum_{z=0}^s \binom{s}{z} \left(\frac{2}{3}\right)^{z}\left(\frac{1}{3}\right)^{s-z} \min \left\{ \frac{1}{2}, \left(\frac{1}{k} + \sqrt{\frac{8}{z} } \right) \right\} \\
    &\le \frac12\sum_{z=0}^{31} \binom{s}{z} \left(\frac{2}{3}\right)^{z}\left(\frac{1}{3}\right)^{s-z} +  \sum_{z=32}^s \binom{s}{z} \left(\frac{2}{3}\right)^{z}\left(\frac{1}{3}\right)^{s-z} \min \left\{ \frac{1}{2}, \left(\frac{1}{k} + \sqrt{\frac{8}{z} } \right) \right\} \\
    &\le \min\left\{\frac12\sum_{z=0}^{s} \binom{s}{z} \left(\frac{2}{3}\right)^{z}\left(\frac{1}{3}\right)^{s-z},  \frac12\sum_{z=0}^{31} \binom{s}{z} \left(\frac{2}{3}\right)^{z}\left(\frac{1}{3}\right)^{s-z }\!\!+ \sum_{z=32}^s \binom{s}{z} \left(\frac{2}{3}\right)^{z}\left(\frac{1}{3}\right)^{s-z} \left(\frac{1}{k} + \sqrt{\frac{8}{z} } \right) \right\} \\
    &\le\min \left\{ \frac{1}{2},  \frac{1}{k}  + \sum_{z=0}^{s/4} \binom{s}{z} \left(\frac{2}{3}\right)^{z}\left(\frac{1}{3}\right)^{s-z} + \sum_{z=s/4}^s \binom{s}{z} \left(\frac{2}{3}\right)^{z}\left(\frac{1}{3}\right)^{s-z} \sqrt{\frac{8}{z} } \right\}.
  \end{align*}
  By the Chernoff bound, we have $\sum_{z=0}^{s/4}\binom{s}{z} (2/3)^{z}(1/3)^{s-z}  \leq  \exp(-s/8)$. Hence the sum is at most
  \[
   \min\left\{\frac{1}{2}, \frac{1}{k} +\exp(-s/8) + \sqrt{\frac{32}{s} } \right\}. \qedhere
  \]
\end{proof}

\begin{lemma}
    Assume that $r$ is an even number in the interval $[d/(\log_3(d)/2),d/\log_3(d)]$. Then the fraction of matrices $Y\in Z^r$ with rank $R(Y)\ge r-\log_3(d)$ that have at least one column independent of the other columns is at least 
\[
  d 2^{2d}  (1/k + 16 \sqrt{\log^2_3(d)/d})^{d-d/\log_3(d)}.
\]
\end{lemma}
\begin{proof}
    Let $Y$ be a uniformly random matrix in $Z^r$ with columns $y_0,\dots,y_{r-1}$. Due to Lemma \ref{lem: lemma for big rank}, we have
    \begin{align*}
        \Pr[\{y_i \in \spn(\{y_j\}_{j \ne i}) \, \forall i\} \cap \{R(Y)=r-s\}] \le \Pr[B_s].
    \end{align*}
    Here $B_s$ is the event that there exists a subset of indices $S \subseteq [r]$ with $|S| \geq r/s$ such that $\dim(\spn(\{y_i\}_{i \in S})) = |S|-1$ but for any subset of $S' \subseteq S$ with $|S'|=|S|-1$, the corresponding vectors are linearly independent. Note that $\Pr[B_i]\le \Pr[B_{i+1}]$ for any $i\in [r-1]$. Thus
    \begin{align*}
            &\Pr[\{y_i \in \spn(\{y_j\}_{j \ne i}) \, \forall i\} \cap \{R(Y)>r-\log_3(d)\}] \\
            & = \sum_{s=1}^{\log_3(d)}\Pr[\{y_i \in \spn(\{y_j\}_{j \ne i}) \, \forall i\} \cap \{R(Y)=r-s\}]\\
            & \le \sum_{s=1}^{\log_3(d)}\Pr[B_s] \le \log_3(d)\Pr[B_{\log_3(d)}].
    \end{align*}
    We must then bound $\Pr[B_{\log_3(d)}]$. On $B_{\log_3(d)}$ there exists a set $S$ of size $|S|\ge r/\log_3(d)$ with the properties mentioned above. Notice that the event $B_{\log_3(d)}$ also implies the existence of a subset of rows $T$ with $|T|=|S|-1$ so that the submatrix $Y_{T,S}$ has full row rank, i.e.\ rank $R(Y_{T,S})\ge r/\log_3(d)-1.$ Now, for any fixed $S\subseteq [r],T\subseteq [d]$ where $|S|\ge r/\log_3(d)$ and $|T|=\log_3(d)-1$, define $E_{S,T}$ as the event that $Y_{T,S}$ has full row rank, and for every subset $S' \subseteq S$ with $|S'|=|S|-1$, the matrix $Y_{S'}$ has full column rank. Then by the above remarks,
    \begin{align*}
        \Pr[B_{\log_3(d)}]\le \Pr\left[\bigcup_{\substack{(S,T)\subseteq [r]\times [d]\\ |S|\ge r/\log_3(d) \\|T|=|S|-1}} E_{S,T}\right] \le \sum_{\substack{(S,T)\subseteq [r]\times [d]\\ |S|\ge r/\log_3(d) \\|T|=\log_3(d)-1}}\Pr[E_{S,T}]
    \end{align*}
Now, observe that for $E_{S,T}$ to occur, we must have that the dimension of the right nullspace of $Y_S$ is $1$. Let $b \in \Z_k^{|S|}$ be the smallest  vector in the lexicographical ordering spanning this space and observe that all its entries are non-zero as otherwise deleting a column corresponding to a zero would result in a set of $|S|-1$ vectors with a linear dependence. 
Any row $x$ of $Y_S$ with index in $[d] \setminus T$ must satisfy $\langle x,b\rangle=0$ for $b$ to be in the right nullspace of $Y_S$.  
Now, observe that $b$ is fixed when conditioning on $Y_{T,S}$ and conditional on $Y_{T,S}$, any row of $Y_S$ with index in $[d] \setminus T$ is uniform on $\{-1,0,1\}^{|S|}$. Thus, since $\Z_k$ is a prime field,  Lemma~\ref{lem:little} implies that 
\[
\Pr[\langle x,b\rangle = 0]=\E\left[\Pr[\langle x,b\rangle = 0|Y_{T,S}]\right] \leq \min\{1/2, 1/k + \exp(-|S|/8)+\sqrt{32/|S|}\}
\]
where $x$ is a row of $Y_S$ with index in $[d] \setminus T$. Since we assumed $r \ge d/\log_3(d)$ and  $|S| \geq r/\log_3(d)$, this probability is at most $1/k + 16 \sqrt{\log_3^2(d)/d}$ for $d$ sufficiently large. Since all such $x$ are independent, we conclude that 
\begin{align*}
    \Pr[E_{S,T}] & \leq \Pr[\langle x,b\rangle = 0, \, \forall x \text{ row in } Y_{S,[d]\setminus T}] \\
    & \le (1/k + 16 \sqrt{\log_3^2(d)/d})^{d-(|S|-1)}\\
    &\le (1/k + 16 \sqrt{\log_3^2(d)/d})^{d-d/\log_3(d)}
\end{align*}

We can now go back and bound
\begin{align*}
    &\Pr[\{y_i \in \spn(\{y_j\}_{j \ne i}) \, \forall i\} \cap \{R(Y)=r-s\}]\le \sum_{\substack{(S,T)\subseteq [r]\times [d]\\ |S|\ge r/\log_3(d) \\|T|=|S|-1}}\Pr[E_{S,T}] \\
    & \le \sum_{\substack{(S,T)\subseteq [r]\times [d]\\ |S|\ge r/\log_3(d) \\|T|=|S|-1}}\Pr[E_{S,T}] (1/k + 16\sqrt{\log_3^2(d)/d})^{d-d/\log_3(d)} \\
    &= \sum_{s=d/\log_3^2(d)}^r \binom{d-d/\log_3(d)}{ s-1}\binom{d}{s} (1/k + 16\sqrt{\log_3^2(d)/d})^{d-d/\log_3(d)}.
\end{align*}
Bounding all binomial coefficients by $2^d$, we finally get that the probability is at most
\[
  d 2^{2d}  (1/k + 16 \sqrt{\log^2_3(d)/d})^{d-d/\log_3(d)}.
\]

\end{proof}
With all our lemmas, we are now ready to bound the moments.
\begin{lemma}[Moment bound]
    Let $r$ an even number in the interval $[d/(2\log_3(d)),d/\log_3(d)]$. Then it holds that
    \[
    \E\left[\left|\sum_{z \in Z} (X_z-\E[X_z])\right|^r\right]\le 113^d|Z|^r\log^d(d)d^{-d/2}
    \]
\end{lemma}
    \begin{proof}

\begin{align*}
    & \E\left[\left|\sum_{z \in Z} (X_z-1/2)\right|^r\right] = \sum_{\substack{Y \in Z^r \\ y_i \in \spn(y_j : i\ne j)\forall i } }\E\left[\prod_{i=0}^{r-1} |X_{y_i}-\E[X_{y_i}]|\right] \\
    & = \sum_{\substack{Y \in Z^r \\ y_i \in \spn(y_j : i\ne j)\forall i \\ R(Y)\ge r-\log_3(d)} }\E\left[\prod_{i=0}^{r-1} |X_{y_i}-\E[X_{y_i}]|\right] +\sum_{\substack{Y \in Z^r \\ y_i \in \spn(y_j : i\ne j)\forall i \\ R(Y) <  r-\log_3(d)} }\E\left[\prod_{i=0}^{r-1} |X_{y_i}-\E[X_{y_i}]|\right] 
\end{align*}
We bound the second sum in the following way
\begin{align*}
    & \sum_{\substack{Y \in Z^r \\ y_i \in \spn(y_j : i\ne j)\forall i \\ R(Y) <  r-\log_3(d)} }\E\left[\prod_{i=0}^{r-1} |X_{y_i}-\E[X_{y_i}]|\right] 
    \le \#\{Y\in Z^r : R(Y)<r-\log_3(d)\} \\
    &  \le \frac{\left(d/2\right)^{-d/2}}{d}|Z|^r,
\end{align*}
by Lemma \ref{lem: small rank}.
Now for the other sum,
\begin{align*}
    &\quad\sum_{\substack{Y \in Z^r \\ y_i \in \spn(y_j : i\ne j)\forall i \\ R(Y)\ge r-\log_3(d)} }\E\left[\prod_{i=0}^{r-1} |X_{y_i}-\E[X_{y_i}]|\right]\\ 
    &\le  \#\{Y \in Z^r :  y_i \in \spn(y_j : i\ne j)\forall i \text{ and }R(Y)\ge r-\log_3(d)\} \\
    &\le 
  d 2^{2d}  (1/k + 12 \sqrt{\log^2_3(d)/d})^{d-d/\log_3(d)}
|Z|^r
\end{align*}
In conclusion the moments are bounded as follows:
\begin{align*}
    & \E\left[\left(\sum_{z \in Z} (X_z-\E[X_z])\right)^r\right] \\
    & \le |Z|^r\left( \frac{\left(d/2\right)^{-d/2}}{d} + d 2^{2d-r}  (1/k + 16 \sqrt{\log^2_3(d)/d})^{d-d/\log_3(d)}  \right)\\
    & \le 5^{d} |Z|^r(1/k + 16 \sqrt{\log^2_3(d)/d})^{d-d/\log_3(d)}\\
    &\leq 5^d|Z|^r\left(17\log(d)/\sqrt{d}\right)^{d(1-1/\log_3(d))}\tag{since $k \geq \sqrt{d}/\log(d)$}\\
    &\leq 150^d|Z|^r\log^d(d)d^{-d/2}
\end{align*}
where we in the last inequality, we use the fact that $d^{\frac{d}{2\log_3(d)}} = (d^\frac{1}{\log_3(d)})^{d/2} = 3^{d/2}$.
\end{proof}
We can now bring everything together and prove \Cref{thm:expansion}.
\begin{proof}[Proof of \Cref{thm:expansion}]
Applying Markov's inequality
\begin{align*}
    p:=\Pr\left[\sum_{z\in Z} X_z \geq (42/50)|Z|\right] &= \Pr\left[\left|\sum_{z\in Z} (X_z-\E[X_z])\right|^r \geq (17/50)^r|Z|^r\right] \\
    &\leq (50/17)^r\frac{150^d|Z|^r\log^d(d)d^{-d/2}}{|Z|^r}\\
    &\leq 450^d\log^d(d)d^{-d/2}.
\end{align*}
We can thus plug this bound on $p$ into \Cref{lem: size of |T|} and conclude that
\[ 
|T|\ge  \left(\frac{\sqrt{d}}{450\log(d)}\right)^{d}.\qedhere
\]
\end{proof}

\subsection{Upper Bounding \texorpdfstring{$|C_f|$}{|C\_f|}}\label{sec:l1_ball_bound}
In this section, we will prove \Cref{lem:Cf_upper}. For this, we will need the following upper bound on the size of discrete $\ell_1$-balls.
\begin{lemma}\label{lem:l1_ball_volume}
    Denote the discrete $\ell_1$-ball with center in the origin as $B_r = \{u \in \Z^d \mid \|u\|_1 \leq r\}$. Then, for $r \geq 1$ it holds that $|B_r| \leq 6^dr^2$.
\end{lemma}
\begin{proof}
    To compute the exact volume of such $\ell_1$-ball, we can sum over all points with distance $i \leq r$ from the origin. This can be done with the stars and bars formula, and then multiplying with the possible number of signs. Using Bernoulli's inequality, we can bound this in the following way.
    \begin{align*}
        |B_r| &\leq \sum_{i=0}^r2^d\binom{d + i - 1}{d-1}
        \leq 2^d \sum_{i=0}^r \left(\frac{e(d+i-1)}{d-1}\right)^{d-1}\\
        &= 2^d e^{d-1} \sum_{i=0}^r \left(1 + \frac{i}{d-1}\right)^{d-1}
        \leq 2^d e^{d-1}\sum_{i=0}^r (1 + i)
        \leq 2^d e^{d-1} \left(r+1 + \frac{r(r+1)}{2}\right)\\
        &\leq 6^{d-1}(r+2)(r+1)
        = 6^{d-1}(r^2 + 3r + 2)
        \leq 6^dr^2.\qedhere
    \end{align*}
\end{proof}
We are now ready to prove \Cref{lem:Cf_upper}. For convenience, we restate the lemma here.
\cfupper*
\begin{proof}
    Fix an $f \in \{0, 1\}^\cX$ and for any $u,v\in\Z_k^d$ let $\nu:\Z_k^d\times \Z_k^d\to \{0,\dots,\lfloor k/2 \rfloor\}$ be the metric defined by $\nu(u,v)=\sum_{i=0}^{d-1}\min(u_i-v_i, v_i-u_i)$ where the minimum is determined by the representatives in $[k]$. Remark that $\nu(u, v)$ can be interpreted as the shortest wrap-around $\ell_1$ distance between $u,v$.
    Now, note that the number of of $x \in \cX$ where $h_u(x) \neq h_v(x)$ is exactly $2\cdot \nu(u, v)$. This means that $\er_{h_u}(h_v) = 2\nu(u, v)/|\cX| = 2\nu(u, v)/(kd)$. Now, fix some $u \in C_f$, and set $r = \eps k d$. Then, denote the discrete wrap-around ball with radius $r$ centered in $u$ as $B_{[r]}(u) = \{v \in \Z_k^d \mid \nu(u,v) \leq r\}$. Also, denote the discrete $\ell_1$ ball with radius $r$ centered in $u$ as $B_r(u) = \{v \in \Z^d \mid \|u-v\|_1 \leq r\}$. Remark that $|B_{[r]}(u)| \leq |B_r(u)|$ since the wrap-around ball starts overlapping with itself when $r \geq \lceil k/2\rceil$.

    Now, assume for sake of contradiction there is a point $v \in C_f$ such that $\nu(u,v) > r$. This would mean that
    \begin{align*}
        \er_{h_u}(h_v)\cdot kd = 2\nu(u,v) > 2r = 2\eps kd \implies \er_{h_u}(h_v) > 2\eps.
    \end{align*}
    However, we know from the definition of $C_f$, that
    \begin{align*}
        \er_{h_u}(h_v)
        &= \Pr_{x\sim \cD}[h_u(x) \neq h_v(x)]\\
        &= \Pr_{x\sim \cD}[(h_u(x) \neq f(x) \land h_v(x) = f(x)) \lor(h_u(x) = f(x) \land h_v(x) \neq f(x))]\\
        &\leq \Pr_{x\sim \cD}[h_u(x) \neq f(x)] + \Pr_{x\sim\cD}[h_v(x) \neq f(x)]\\
        &\leq 2\eps
    \end{align*}
    giving us a contradiction. Therefore, for every $v \in C_f$ we know that $\nu(u,v) \leq r$. This means that $C_f \subseteq B_{[r]}(x)$, meaning that $|C_f| \leq |B_{[r]}(x)| \leq |B_r(x)| \leq 6^dr^2$ where the last inequality follows from \Cref{lem:l1_ball_volume}. Note, by the definition of $k$, we have $\eps k \geq 2$. Therefore, when we plug in the value of $r$, we get
    \begin{align*}
        |C_f| \leq 6^d (\eps kd)^2 \leq 6^d(\eps k)^d 2^{2-d}d^2
        \leq (6\eps k)^d
    \end{align*}
    where the last inequality holds for $d \geq 8$.
\end{proof}

\section{Proof of the Upper Bound}

In this section, we will establish the upper bound on the sample complexity for replicably PAC learning the hypothesis class from above in the realizable setting, which nearly matches our lower bound.
For convenience, we restate the theorem here.
\upperbound*

We recall that the input domain is $\cX = [d] \times \Z_k$ and the distribution $\cD$ is uniform over $\cX$. Our hypothesis set $\cH$ contains a hypothesis $h_{i}$ for every $d$-tuple $i = (i_0,\dots,i_{d-1}) \in \Z_k^d$. For a point $(a,b) \in \cX$, we have $h_i((a,b)) = 1$ if $i_a \leq b < i_a + \lfloor k/2\rfloor$ or if $b < i_a + \lfloor k/2\rfloor < i_a$. Otherwise $h_i((a,b))=0$. 

Let $h^\star=h_{i^\star} \in \cH$ be an unknown target function and assume training samples are drawn by sampling $x_1,\dots,x_n$ i.i.d.\ from $\cD$ and constructing the training set $(x_1,h^\star(x_1)),\dots,(x_n,h^\star(x_n))$. Let a desired accuracy $0<\eps<1$, replicability parameter $0<\rho<1$ and failure probability $0<\delta<1$ be given. If $k = O(\eps^{-1} \rho^{-1} \sqrt{d})$, then with $n = \widetilde{O}(dk) = \widetilde{O}(\eps^{-1} \rho^{-1} d^{3/2})$ samples, we will see every point in the input domain with probability at least $1-\delta$. We can thus output the unique $h^\star$. This is also replicable. So we assume $k \geq C \eps^{-1} \rho^{-1} \sqrt{d}$ for sufficiently large constant $C>0$. 

Our learning algorithm is as follows: On samples $S=\{(x_i,h^\star(x_i))\}_{i=1}^n$ with $x_i = (a_i,B^z_i)$, partition $S$ into $d$ pieces $S_0,\dots,S_{d-1}$ such that $S_a$ contains all samples $(a_i,B^z_i)$ with $a_i=a$. For each $S_a$, sort the samples by $B^z_i$ and remove duplicates. Let $b^S_a$ denote the value $B^z_i$ of the point whose predecessor $b_j$ (possibly with wrap around) has $h^\star((a,b_j))=0$ while $h^\star((a,B^z_i))=1$. If some $S_a$ contains no $0$'s or no $1$'s, we simply let $b^S_a=0$. We have that $b^S_a$ serves as an estimate of $i^\star_a$. 

Now, as also introduced in the technical overview section, we define $\nu(a, b)$ for $a,b \in \Z_k$ as the distance between $a$ and $b$ with wrap-around, i.e.\ $\nu(a,b) = \min\{a-b,b-a\}$, where the minimum is determined when treating $a,b$ as elements of $\Z_k$. That is, we apply $\bmod k$ before taking minimum. Let $0 < \beta < \eps/4$ be a parameter to be determined. Shuffle all hypotheses in $\cH$ (using the shared randomness) and return the first hypothesis $h_S:=h_i$ satisfying $\sum_{a=0}^{d-1} \nu(b^S_a, i_a) \leq \eps kd/4$.

\paragraph{Correctness.} First observe that $h_S$ and the hypothesis $h_{b^S}$ with $b^S=(b^S_0,\dots,b^S_{d-1})$ disagree in the predictions of at most $\eps k d/2$ points. Since $\cD$ is uniform over $dk$ points, we have $|\er_\cD(h_S) - \er_{\cD}(h_{b^S})| \leq \eps /2$. Next we show that $b^S_a$ is close to $i^\star_a$ for all $a$:

\begin{lemma}
\label{lem:close}
For any $0 < \alpha < 1/2$ and any $0 < \delta < 1$, if $n \geq \alpha^{-1} d \ln(2d/\delta)$ then it holds with probability at least $1-\delta$ that $\nu(b^S_a, i^\star_a) \leq \alpha k$ for all $a \in [d]$.
\end{lemma}

\begin{proof}
Fix a coordinate $a$. If $S_a$ contains at least one point with label $0$ and at least one point of the form $(a,B^z_i)$ with $B^z_i \in \{i^\star_a,\dots,i^\star_a + \alpha k\}$ then $\nu(b^S_a, i^\star_a) \leq \alpha k$. The probability that $S_a$ contains no points with label $0$ is at most $(1-1/(2d))^n \leq \exp(-n/(2d))$. The probability that $S_a$ contains no points $(a,B^z_i)$ with $B^z_i$ of the above form is at most $(1-(\alpha k+1)/(d k))^n \leq \exp(-\alpha n/d)$. For $n \geq \alpha^{-1} d \ln(2d/\delta)$ we can union bound over all $d$ choices of $a$ to conclude $\nu(b^S_a, i^\star_a) \leq \alpha k$ for all $a$.
\end{proof}

If we pick $\alpha=\eps/4$ and require $n = \Omega(\eps^{-1}d \ln(d/\delta))$, we get with probability at least $1-\delta$ that $\sum_{a=0}^{d-1} \nu(b^S_a, i^\star_a) \leq \eps d k/4$. This implies $\er_\cD(h_{b^S})=\er_\cD(h_{b^S}) - \er_\cD(h^\star) \leq \eps/2$. Using the triangle inequality, we conclude $\er_\cD(h_S) \leq \eps$ with probability $1-\delta$.

\paragraph{Replicability.} 
For the replicability guarantee, let $r$ characterize the random shuffle of all hypotheses in $\cH$. For a fixed $r$, for two arbitrary samples $S, S'$, we say that $S$ (respectively $S'$) \emph{accepts} the $i$'th hypothesis if the $i$'th hypothesis in the shuffled $\cH$ 
satisfies $\sum_{a=0}^{d-1}\nu(b_a^S,i_a) \leq \eps k d/4$. Abusing notation slightly, let the $i^{\text{th}}$ hypothesis in $\cH$ (when $\cH$ is shuffled according to randomness $r$) be $h_i$, corresponding to the tuple $i=(i_0,\dots,i_{d-1})$.

Let $A_{i,r}$ (respectively $A'_{i,r}$) denote the event that the $i$'th hypothesis (in the shuffle) is the \textit{first} hypothesis in the shuffled $\cH$ to be accepted by $S$ (respectively $S'$). Similarly, let $B_{i,r}$ (respectively $B'_{i,r}$) denote the event that the $i$'th hypothesis in the shuffle is accepted by $S$ (respectively $S'$) (note that $A_{i,r} \subseteq B_{i,r}$ but not necessarily vice-versa). Let $E$ denote the event that $b^S$ and $b^{S'}$ satisfy $\nu(b^S_a, b^{S'}_a) \leq \beta k/2$ for all $a$. Note that this event is defined solely from $S$ and $S'$ and is independent of the randomness $r$.

By relating both $b^S$ and $b^{S'}$ to $i^\star$ and instantiating the triangle inequality, the correctness proof above (Lemma~\ref{lem:close}) gives us such $b^S$ and $b^{S'}$ with probability $1-\rho/4$ when $n \geq c \beta^{-1}d \ln(d/\rho)$ for large enough constant $c$.

Now consider two arbitrary samples $S,S'$. The event $\cA(S;r) \neq \cA(S';r)$ implies that $\exists i \in \{1,2,\dots,|\cH|\}$ such that the event $A_{i,r}$ occurs but $B'_{i,r}$ does not, \textit{or} the event $A'_{i,r}$ occurs but $B_{i,r}$ does not. Therefore, we have
\begin{align*}
    \Pr_{S,S',r}[\cA(S,r) \neq \cA(S',r)] &\le \Pr_{S,S',r}[\cA(S,r) \neq \cA(S',r) \mid E] + \rho/4 \\
    &\le \Pr\left[\exists i \in \{1,\dots,|\cH|\}: (A_{i,r} \land \neg B'_{i,r}) \lor (A'_{i,r} \land \neg B_{i,r}) \mid E\right] + \rho/4\\
    &\le \sum_{i=1}^{|\cH|}\Pr\left[A_{i,r} \land \neg B'_{i,r}\mid E\right] + \sum_{i=1}^{|\cH|}\Pr\left[A'_{i,r} \land \neg B_{i,r}\mid E\right]+ \rho/4\\
    &= \sum_{i=1}^{|\cH|}2\Pr\left[A_{i,r} \land \neg B'_{i,r}\mid E\right] + \rho/4\\
    &= \sum_{i=1}^{|\cH|} 2 \Pr_{S,S',r}[A_{i,r} \mid E] \Pr[\neg B'_{i,r} \mid A_{i,r}, E] + \rho/4.
\end{align*}
In the following, we will show that for any pair of samples $S,S'$ satisfying the conditions in the event $E$ (i.e.\ $\nu(b_a^S,b_a^{S'}) \leq \beta k/2$ for all $a$), we have $\Pr_r[\neg B'_{i,r} \mid A_{i,r}] \leq 3\rho/8$ for all $i$. Combining this with $\sum_i \Pr[A_{i,r} \mid E] = 1$ gives the required replicability guarantee. So fix an arbitrary such pair of samples $S, S'$.


Conditioning on the event $A_{i,r}$ implies that $\mathcal{A}(S,r)=h_i$ and also that the distribution of $i$ is uniform random among all $i$ satisfying $\sum_{a=0}^{d-1}\nu(b_a^S,i_a) \leq \eps k d/4$. 
For each coordinate $a \in [d]$, let $\sigma_a \in \{-1,1\}$ be so that $i_a + \sigma_a \nu(b_a^S, i_a) k = b_a^S$ (picking $\sigma_a$ uniformly in case of ties). 

Let us now consider the distance $\sum_{a=0}^{d-1} \nu(b_a^{S'}, i_a)$. We want to show that this distance is no more than $\eps k d/4$ with large probability, i.e.\ $S'$ also accepts $h_i$.

Recall that for every coordinate $a$, we have that $i_a = b_a^S - \sigma_a \nu(b_a^S, i_a)$. It follows that if $\sigma_a$ is such that $b_a^S - \sigma_a \nu(b_a^S,b_a^{S'})= b_a^{S'}$, then 
\[
    \nu(b_a^{S'}, i_a) = \max\{\nu(b_a^S,b_a^{S'}), \nu(b_a^S, i_a)\} - \min\{\nu(b_a^S,b_a^{S'}), \nu(b_a^S, i_a)\}.
\]
Otherwise, we have 
\[
    \nu(b_a^{S'}, i_a) \leq \nu(b_a^S,b_a^{S'})+\nu(b_a^S, i_a) =\max\{\nu(b_a^S,b_a^{S'}), \nu(b_a^S, i_a)\} + \min\{\nu(b_a^S,b_a^{S'}), \nu(b_a^S, i_a)\}. 
\]
Since either of these cases happens with probability $1/2$ each, and using $\nu(b_a^S,b_a^{S'}) \leq \beta k /2$ under the event $E$, we have
\begin{align*}
&\Pr\left[\sum_{a=0}^{d-1}\nu(b_a^{S'},i_a) > \eps k d/4\right]  \\
& \leq\Pr\left[\sum_{a=0}^{d-1} \max\{\nu(b_a^S,b_a^{S'}), \nu(b_a^S, i_a)\} + \tau_a \min\{\nu(b_a^S,b_a^{S'}), \nu(b_a^S, i_a)\} > \eps k d /4\right]  \\
& \leq\Pr\left[\sum_{a=0}^{d-1} \max\{\beta k /2, \nu(b_a^S, i_a)\} + \tau_a \min\{\nu(b_a^S,b_a^{S'}), \nu(b_a^S, i_a)\} > \eps k d /4\right]  \\
&=\Pr\left[ \sum_{a=0}^{d-1} \tau_a \min\{\nu(b_a^S,b_a^{S'}), \nu(b_a^S, i_a)\} > \eps k d /4 - \sum_{a=0}^{d-1} \max\{\beta k /2, \nu(b_a^S, i_a)\}\right].
\end{align*}
where the $\tau_a$'s are uniformly random signs that may be sampled independently of everything. 

Note that if we condition on everything but the signs $\tau_a$, then by Hoeffding's inequality and the fact that $\nu(b_a^S,b_a^{S'}) \leq \beta k /2$, we have 
\begin{align*}
\Pr\left[\sum_{a=0}^{d-1}\tau_a\min\{\nu(b_a^S,b_a^{S'}), \nu(b_a^S, i_a)\} > t\right] & < \exp\left(\frac{-2 t^2}{\sum_{a=0}^{d-1} 4 \min\{\nu(b_a^S,b_a^{S'}), \nu(b_a^S, i_a)\}^2} \right) \\
&\leq \exp\left( \frac{-2 t^2}{d \beta^2 k^2} \right).
\end{align*}
We will pick $t=\eps k d /4 - \sum_{a=0}^{d-1} \max\{\beta k /2, \nu(b_a^S, i_a)\}$ and therefore we set out to upper bound $\sum_{a=0}^{d-1} \max\{\beta k /2, \nu(b_a^S, i_a)\}$. Our goal is to show that $t \geq \beta k \sqrt{d \ln(2/\rho)}$ with high probability. When this is the case, it holds with probability at least $1-\rho/4$ over the signs $\tau_a$ that $\sum_{a=0}^{d-1} \nu(b^{S'}_a, i_a) \leq \eps k d/4$.

Let $\Delta$ denote the vector with coordinates $\nu(b_a^S,i_a) \sigma_a$ and observe that $\Delta$ is uniform random among all vectors in $\Z^d$ with $\|\Delta\|_\infty \leq k/2$ (due to wrap around) and $\|\Delta\|_1 \leq \eps k d /4$. Our goal is to show that $\sum_{a=0}^{d-1} \max\{\beta k/2, |\Delta_a|\}$ is noticeably smaller than $\eps k d/4$ with high probability. Here we first observe that
\[
\sum_{a=0}^{d-1} \max\{\beta k/2, |\Delta_a|\} \leq \|\Delta\|_1 + |\{ a : |\Delta_a| < \beta k/2\}| \cdot \beta k/2.
\]
We bound the two terms below and arrive at the following two technical results.
\begin{lemma}
\label{lem:smallnorm}
Assume $k \geq 384 \eps^{-1} \rho^{-1}$. Then for any $\rho/4 \leq \gamma \leq 1/2$, it holds with probability at least $1-\gamma$ that $\|\Delta\|_1 \leq \eps dk/4 - \eps \gamma k/96$.
\end{lemma}

\begin{lemma}
\label{lem:fewsmall}
If $\beta \geq 2/k$ and $k \geq 384 \eps^{-1} \rho^{-1}$, then it holds with probability at least $1-\rho/4$ that $|\{ a : |\Delta_a| < \beta k/2\}| \leq 2304 d \beta \eps^{-1} + 3 \ln(4/\rho)$.
\end{lemma}

Invoking Lemma~\ref{lem:smallnorm} with $\gamma=\rho/4$ and using Lemma~\ref{lem:fewsmall}, we get that with probability at least $1-\rho/2$, we have
\begin{align*}
    \sum_{a=0}^{d-1} \max\{\beta k/2, |\Delta_a|\} &\leq \|\Delta\|_1 + |\{ a : |\Delta_a| < \beta k/2\}| \cdot \beta k/2 \\
    &\leq \eps d k/4 - \eps \rho k/(4 \cdot 96) + (\beta k/2)(2304 d \beta \eps^{-1} + 3 \ln(4/\rho)).
\end{align*}
Let us now set $\beta = c \min\{\eps \rho/\sqrt{d \ln(2/\rho)}, \eps \rho/\ln(4/\rho)\}$ for sufficiently small constant $c>0$. For $c$ small enough, we then have $\beta^2 k d \eps^{-1}/1152 \leq \eps \rho k/(16 \cdot 96)$ and $(3/2)\beta k \ln(4/\rho) \leq \eps \rho k/(16 \cdot 96)$. This implies that
\begin{align*}
\sum_{a=0}^{d-1} \max\{\beta k/2, |\Delta_a|\} &\leq \eps d k/4 - \eps \rho k/(8 \cdot 96).
\end{align*}
Recall from above that we picked $t = \eps k d/4 - \sum_{a=0}^{d-1} \max\{\beta k/2, |\Delta_a|\}$. We therefore have $t \geq \eps \rho k/(8 \cdot 96)$. We needed this to satisfy $t \geq \beta k \sqrt{d \ln(2/\rho)}$. This is indeed satisfied whenever $\eps \rho/(8 \cdot 96) \geq \beta \sqrt{d \ln(2/\rho)}$. Our choice of $\beta$ satisfies this for $c$ small enough.

From earlier, we had that the sample complexity was $n = O(\beta^{-1} d  \ln(d/\rho))$. Inserting $\beta$ finally gives a sample complexity of
\[
n = \widetilde{O}\left( \eps^{-1} \rho^{-1} d^{3/2} \right) = \widetilde{O}\left( \eps^{-1} \rho^{-1} (\log |\cH|)^{3/2} \right) .
\]

\paragraph{Bounding $\ell_1$-Norm.}
In the following we prove Lemma~\ref{lem:smallnorm}.
Define events $F_0,\dots,F_d$, where $F_z$ is the event that $\Delta$ has precisely $z$ entries that are zero. Then for any $t$
\[
\Pr[\|\Delta\|_1 \leq t] = \sum_{z=0}^{d} \Pr[\|\Delta\|_1 \leq t \mid F_z] \Pr[F_z]
\]
We will first bound $\Pr[\|\Delta\|_1 \leq t \mid F_z]$. To simplify this analysis, we will relate the distribution of $\|\Delta\|_1$ conditioned on $F_z$ to another random variable $\Delta^z$ with a slightly simpler distribution. Concretely, let $\Delta^z$ be sampled uniformly among all vectors $v$ in $\Z^{d}$ with $\|v\|_1 \leq \eps k d/4$ and precisely $z$ entries that are $0$. That is, we drop the requirement $|v_a| \leq k/2$ compared to the distribution of $\Delta$ conditioned on $F_z$. We claim that 
\begin{lemma}
\label{lem:coupling}
For any $t$, we have $\Pr[\|\Delta\|_1 \leq t \mid F_z] \geq \Pr[\|\Delta^z\|_1 \leq t]$.
\end{lemma}

\begin{proof}
Let $p_i = \Pr[\|\Delta^z\|_\infty \leq k/2 \mid \|\Delta^z\|_1 = i]$. The $p_i$'s are monotonically decreasing in $i$. To see this, for any $i \le \eps k d/4$, let $B^z_i := \{v:\|v\|_1 = i, \|v\|_0 = d-z\}$, where $\|v\|_0$ denotes the number of non-zero entries of $v$. Then, we have that $\Pr_{\Delta^z}[\Delta^z=v ~|~\|\Delta^z\|_1=i]=1/|B^z_i|$. Let $\mu_i$ be the uniform distribution over $B^z_i$. Notice also that $|B^z_i|=2^{d-z} \cdot \binom{d}{z} \cdot \binom{i-(d-z)+d-z-1}{d-z-1} = 2^{d-z} \cdot \binom{d}{z} \cdot \binom{i-1}{d-z-1}$: this is the number of ways we can pick $z$ entries to be zero, a sum of $d-z$ integers, each of which is at least $1$, to obtain the sum $i$, and then assigning signs to all the integers.

Now 
consider the randomized map $\rho: \Z^d \to \Z^d$ that acts as follows: on any input $v$, $\rho$ first samples a non-zero coordinate $j$ with probability $\frac{|v_j|-1}{\|v\|_1 - (d-z)}$, and then outputs $\tilde{v}$, where, 
\begin{align*}
    \tilde{v}_{j'}=\begin{cases} v_{j'} & \forall j' \neq j \\
    \text{$|v_j|-1$ with probability 1/2, and $-(|v_j|-1)$ with probability 1/2} & \text{for $j'=j$}.
    \end{cases}
\end{align*}
That is, $\tilde{v}$ is equal to $v$ at all coordinates other than $j$, where its absolute value is one smaller than $|v_j|$, so that $\|\tilde{v}\|_1 = \|v\|_1-1$. Then, if we first sample $v \sim \mu_i$, and then obtain $\tilde{v}=\rho(v)$, observe that the probability of obtaining a particular $\tilde{v} \in B^z_{i-1}$ is precisely the chance that we sampled $v$ that satisfies $|v_j|=|\tilde{v}_j|+1$ for some $j$ with $v_j \neq 0$ and $v_{j'} = \tilde{v}_{j'}$ for all $j' \neq j$, and thereafter sampled the coordinate $j$ and set $\tilde{v}_j$ as required: this is equal to
\begin{align*}
    \sum_{j : \tilde{v}_j \neq 0}\frac{1}{|B^z_i|} \cdot \frac{|\tilde{v}_j|+1-1}{i-(d-z)}\cdot \frac{1}{2} \cdot 2 = \frac{i-1}{|B^z_i|(i-(d-z))} = \frac{1}{|B^z_{i-1}|},
\end{align*}
where we used the expression $|B^z_i|=2^{d-z} \cdot \binom{d}{z} \cdot \binom{i-1}{d-z-1}$. We have thus argued that we can obtain a sample from $\mu_{i-1}$ by first sampling $v \sim \mu_i$, and then applying $\rho(v)$. We can then conclude 
\begin{align*}
    p_{i-1} &= \Pr[\|\Delta^z\|_\infty \leq k/2 \mid \|\Delta^z\|_1 = i-1] = \Pr_{\Delta' \sim \mu_{i-1}}[\|\Delta'\|_\infty \leq k/2 ] \\
    &= \Pr_{\Delta' \sim \mu_{i}}[\|\rho(\Delta')\|_\infty \leq k/2 ]
    \ge \Pr_{\Delta' \sim \mu_{i}}[\|\Delta'\|_\infty \leq k/2 ]\\
    &= \Pr[\|\Delta^z\|_\infty \leq k/2 \mid \|\Delta^z\|_1 = i]=p_i,
\end{align*}
where the inequality above follows because $\|\Delta'\|_\infty \le k/2 \implies \|\rho(\Delta')\|_\infty \le k/2$. This establishes that the $p_i$'
s are non-increasing.

Now observe that the distribution of $\Delta^z$ conditioned on $\|\Delta^z\|_\infty \leq k/2$ equals the distribution of $\Delta$ conditioned on $F_z$. We thus have
\begin{align*}
    \Pr[\|\Delta\|_1 \leq t \mid F_z] &= \sum_{i=0}^t \Pr[\|\Delta\|_1 = i \mid F_z] \\
    &= \frac{\sum_{i=0}^t \Pr[\|\Delta^z\|_1 = i] p_i}{\Pr[\|\Delta^z\|_\infty \leq k/2]} \\
    &= \frac{\sum_{i=0}^t \Pr[\|\Delta^z\|_1 = i] p_i}{\sum_{i=0}^{\eps k d/4} \Pr[\|\Delta^z\|_1 = i] p_i } \\
    &\geq \frac{\sum_{i=0}^t \Pr[\|\Delta^z\|_1 = i] p_i}{\sum_{i=0}^{t} \Pr[\|\Delta^z\|_1 = i] p_i  +\sum_{i=t+1}^{\eps k d/4} \Pr[\|\Delta^z\|_1 = i] p_t  } \\
    &= 1 - \frac{\sum_{i=t+1}^{\eps k d/4} \Pr[\|\Delta^z\|_1 = i] p_t}{\sum_{i=0}^{t} \Pr[\|\Delta^z\|_1 = i] p_i  +\sum_{i=t+1}^{\eps k d/4} \Pr[\|\Delta^z\|_1 = i] p_t  } \\
    &\geq 1 - \frac{\sum_{i=t+1}^{\eps k d/4} \Pr[\|\Delta^z\|_1 = i] p_t}{\sum_{i=0}^{t} \Pr[\|\Delta^z\|_1 = i] p_t  +\sum_{i=t+1}^{\eps k d/4} \Pr[\|\Delta^z\|_1 = i] p_t  } \\
    &= 1 - \Pr[\|\Delta^z\|_1 > t] \\
    &= \Pr[\|\Delta^z\|_1\leq t].
\end{align*}
\end{proof}

In light of Lemma~\ref{lem:coupling} we set out to show that $\|\Delta^z\|_1$ is somewhat smaller than $\eps d k/4$ with high probability.

We see for $t \geq 2d$ that
\[
1 \leq \frac{B^z_{t+1}}{B^z_t} = \frac{\binom{t+1}{d-z-1}}{\binom{t}{d-z-1}} = \frac{t+1}{t+1-(d-z-1)} = 1 + \frac{d-z-1}{t+2-(d-z)} \leq 1 + \frac{2d}{t}.
\]
Here the last inequality follows from $d-z \leq d \leq t/2$ (assuming $t \geq 2d$).
Assume now that $k \geq 384 \eps^{-1} \rho^{-1}$. For any $\gamma$ satisfying $\rho/4 \leq \gamma \leq 1/2$, define $q = \eps \gamma k/96$ (which is at least $1$ by our requirement on $k$). Note that this choice of $q$ also satisfies $\eps kd/4-2 q \gamma^{-1} \geq \eps k d/8 \geq 2d$. We now have that the probability that $\|\Delta^z\|_1 > \eps k d/4 - q$ is at most
\begin{align*}
\frac{\sum_{t=\eps k d/4-q+1}^{\eps k d/4} B^z_t}{\sum_{t=0}^{\eps k d/4} B^z_t} &\leq
\frac{\sum_{t=\eps k d/4-q+1}^{\eps k d/4} B^z_t}{\sum_{t=\eps k d/4- 2 q \gamma^{-1}}^{\eps k d/4} B^z_t} 
\\&\leq
\frac{q \cdot B_{\eps k d/4-q+1}(1+2d/(\eps k d/8))^{q}}{(2 \gamma^{-1} q) \cdot B_{\eps k d/4-q+1}(1+2d/(\eps k d/8))^{-2 \gamma^{-1} q}} \\&\leq
\frac{(1 + 2d/(\eps k d/8))^{q + 2 \gamma^{-1} q}}{2 \gamma^{-1}} \\&\leq
\frac{\exp(48 \gamma^{-1} q/(\eps k))}{2 \gamma^{-1}} \\&\leq
e^{1/2} \gamma/2 \\&\leq
\gamma.
\end{align*}
We thus have with probability at least $1-\gamma$ that $\|\Delta^z\|_1 \leq \eps d k /4 - \eps \gamma k/96$. Combined with Lemma~\ref{lem:coupling} this also implies
\begin{align*}
    \Pr[\|\Delta\|_1 \leq \eps d k/4 - \eps \gamma k /96] &= \sum_{z=0}^{d} \Pr[\|\Delta\|_1 \leq \eps d k/4 - \eps \gamma k /96 \mid F_z] \Pr[F_z] \\
    &\geq \Pr[\|\Delta^z\|_1 \leq \eps d k/4 - \eps \gamma k /96 ] \Pr[F_z] \\
    &\geq 1-\gamma.
\end{align*}
This completes the proof of Lemma~\ref{lem:smallnorm}.

\paragraph{Bounding Number of Small Coordinates.}
We next set out to prove that $|\{a : |\Delta_a| < \beta k/2\}|$ is small with high probability.

\begin{proof}[Proof of Lemma~\ref{lem:fewsmall}]
Let $X_a$ denote an indicator random variable for the event $|\Delta_a| \leq \beta k/2$. We now bound $\Pr[X_a=1]$. For this, consider the set $S$ of all $v$ with $\|v\|_\infty \leq k/2$ and $\|v\|_1 \leq \eps kd/4$. Let $S_a$ be the set of all $v$ with $\|v\|_\infty \leq k/2,\|v\|_1 \leq \eps k d/4$ and $|v_a| \leq \beta k/2$. By definition, we have $|S_a| = \Pr[X_a=1]|S|$. 

Now observe that the events $|v_a| \leq \beta k/2$ and $\|v\|_1 \leq \eps kd/4 - q$ for $q \geq 1$ are positively correlated. Let us now pick $q=\eps k/192$. Then by Lemma~\ref{lem:smallnorm}, we have $\Pr[\|\Delta\|_1 \leq \eps d k/4 - q] \geq 1/2$. By the positive correlation, this further implies that the subset $S^\star_a \subseteq S_a$ of vectors $v$ in $S_a$ also satisfying $\|v\|_1 \leq \eps k d/4 - q$ has $|S^\star_a| \geq |S_a|/2 = \Pr[X_a=1]|S|/2$.

We now relate $|S^\star_a|$ and $|S|$ by considering a bipartite graph where the left side has a node for each $v \in S^\star_a$ and the right side has a node for each $v \in S$. For each $v \in S^\star_a$, add an edge to every $w \in S$ so that $v_j = w_j$ for all $j \neq a$. We argue that every node on the left side has large degree, and every node on the right side has small degree. This eventually bounds the ratio between the number of nodes on the two sides.

So consider a node on the left side, corresponding to a fixed $v \in S^\star_a$. By definition of $S^\star_a$, we have $\|v\|_1 \leq \eps k d/4 - q$ and $|v_a| \leq \beta k/2$. This implies that every  vector $w$ with $|w_a| \leq q - \beta k/2$ and $w_j=v_j$ for $j \neq a$ has $\|w\|_1 \leq \eps k d/4$ and thus any such $w$ is in $S$. The degree of $v$ in $S^\star_a$ is hence at least $2(q-\beta k/2)$.

Consider next a node on the right side, corresponding to a fixed $w \in S$. For any integer $z$ with $|z| \leq \beta k /2$, there is at most one $v \in S^\star_a$ satisfying $w_j = v_j$ for every $j \neq a$ and $v_a = z$. Thus the degree of $w$ is at most $\beta k+1$.

If $E$ denotes the set of edges in the bipartite graph, we thus have $|E| \geq |S^\star_a|(2q - \beta k)$ and $|E| \leq |S|(\beta k+1)$. We therefore have $|S^\star_a| \leq |S|(\beta k+1)/(2q-\beta k)$. Combining this with the inequality $|S^\star_a|\geq \Pr[X_a=1]|S|/2$ finally yields $\Pr[X_a=1] \leq 2(\beta k+1)/(2q-\beta k)$.

If we constrain $\beta k \leq q = \eps k/192$, this probability is at most $2(\beta k+1)/q = 384 (\beta k+1) \eps^{-1}/k$. If we further require $\beta k \geq 1$, this is again upper bounded by $768 \beta \eps^{-1}$.
Since the $X_a$ are negatively correlated, we have by a Chernoff bound that $\Pr[\sum_a X_a > 2304 d \beta \eps^{-1}+ 3\ln(4/\rho)] \leq \rho/4$.
\end{proof}

\appendix
\section{Deferred Proofs}\label{sec:appendix}
In this section, we will prove all the lemmas that were skipped in the main text.
For convenience, we restate the lemmas.

\begin{lemma}[Independence and linear independence] \label{lem: independence}
    Let $k$ be prime power, $d\in \mathbb{N}$ and consider the finite field with $k$ elements $\F_k$. Then let $y_1,\dots,y_r$ be vectors from $\F_k^d$ and let $v$ be uniform on $\F_k^d$. If $y_1\notin \emph{Span}\{y_2,\dots,y_r\}$, then $\langle y_1,v\rangle$ is independent of $\left(\langle y_2,v\rangle, \dots, \langle y_r,v\rangle \right)$.
\end{lemma}

\begin{proof}
    For a finite set $S$, let $\cU(S)$ denote the uniform distribution on $S$. Also, for a matrix $M$, let $r(M)$ be the rowspace of $M.$
    Let $A$ denote the matrix with $y_1,\dots,y_r$ as rows and $B$ the matrix with $y_2,\dots,y_r$ as rows. Note then that $Av=\left(\langle y_1,v\rangle, \dots, \langle y_r,v\rangle \right)$. We first show that $Av$ is uniform on $r(A)$, the row space of $A$. Let $x,y\in r(A)$. Then there is $w\in \F_k^d$ such that $Aw=x-y$, and hence
    \begin{align*}
        \Pr[Av=x]=\Pr[A(v+w)=x] = \Pr[Av=x-Aw]=\Pr[Av=y],
    \end{align*}
    which shows the first claim, using that $v$ is uniform. Now, by linear independence, we must have (for example by a dimension-argument) that 
    $r(A) = \F_k \times r(B)$. 
    Thus, we have the joint distribution: 
    \[
    Av = (\langle y_1,v\rangle, Bv) \sim \mathcal{U}(\F_k \times r(B)) = \mathcal{U}(\F_k)\otimes \mathcal{U}(r(B)).
    \]
    In particular $\langle y_1,v\rangle$ is independent of $Bv=\left(\langle y_2,v\rangle, \dots, \langle y_r,v\rangle \right)$.
\end{proof}

\begin{lemma}\label{lem:remove_recursion}
    Let $x > 0$ and $\alpha \geq e$ be positive reals. Then $x \geq \alpha/\log(x)$ implies $x \geq \alpha/\log(\alpha)$.
\end{lemma}
\begin{proof}
    Assume for sake of contradiction that $x < \alpha / \log \alpha$.
    Then
    \begin{align*}
        x \geq \alpha / \log x
        > \frac{\alpha}{\log(\alpha / \log(\alpha))} \geq \alpha/\log \alpha
    \end{align*}
    giving us the desired contradiction.
\end{proof}

\majorization*
\begin{proof}
    The proof goes by induction in $d$. For the base case $d=0$, one can just set $p_{0,0} = 1$; the first and second properties are immediately satisfied, while the third property follows from the fact that $x_0 = y_0$.
    
    For the inductive step, we can assume, that one can always find such values $p_{i,j}$ for instances of size $d-1$.
    Now assume we are given an instance $x_0,\dots,x_d$ and $y_0,\dots,y_d$ of size $d$. Remark that, if $x_d = 0$, then it follows from the dominating and equal sum property that $y_d = 0$. Therefore, we can just pick $p_{d,j} = 1/d$ for all $j \in \{0, \dots, d\}$ and then solve the problem for $x_1,\dots,x_{d-1}$ and $y_1,\dots,y_{d-1}$ using the induction hypothesis. We can verify that this satisfies all the three properties required of the $p_{i,j}$ values.
    
    We now move on to the case of $x_d \neq 0$. Here, we will pick the values $p_{d,k}$ in a greedy fashion, where $p_{d,k}$ is chosen in terms of $p_{d,k+1}, \dots, p_{d, d}$ as seen below: 
    \begin{align*}
        p_{d,d} &= \frac{y_d}{x_d},  &&\forall k \in \{0, \dots, d-1\}: p_{d,k} = \min\left\{\frac{y_k}{x_d}, 1 - \sum_{j=k+1}^d p_{d,j}\right\}.
    \end{align*}

    First, we verify that the $p_{d,k}$ values satisfy Property \ref{prop:non_neg} and \ref{prop:sum} above. Note that since $\sum_{i=0}^{d-1}x_i \le \sum_{j=0}^{d-1}y_j$ and $\sum_{i=0}^{d}x_i=\sum_{j=0}^d y_j$, it must be the case that $y_d \le x_d$, which means that $0 \le p_{d,d} \le 1$. For any other $k \neq d$, observe that by definition, $p_{d,k} \le 1-\sum_{j=k+1}^dp_{d,j} \implies \sum_{j=k}^d p_{d,j} \le 1$. Together, we have that $\sum_{j=k}^d p_{d,k} \le 1$ for every $k \in \{0,\dots,d\}$. This means that $ 0 \le p_{d,k} \le 1$ for all $k$.

    For Property \ref{prop:sum}, we claim that there must be some $k \in \{0,\dots,d-1\}$ for which $p_{d,k}=1-\sum_{j=k+1}^d p_{d,j}$. Otherwise, if $p_{d,k} = y_k/x_d$ for all $k = \{1, \dots, d\}$, then
    \begin{align*}
        \sum_{j=0}^d y_j \geq \sum_{i=0}^dx_i \ge x_d \implies \frac{y_0}{x_d} \geq 1 - \sum_{j=1}^d \frac{y_j}{x_d} = 1 - \sum_{j=1}^d p_{d,j},
    \end{align*}
    meaning that $p_{d,0} = \frac{y_0}{x_d} = 1 - \sum_{j=1}^d p_{d,j}$.
    So we know that there is some value $k\in\{0, \dots d-1\}$ where $p_{d,k} = 1 - \sum_{j=k+1}^d p_{d,j}$. Then it must be the case that $p_{d,l} = 0$ for all $l < k$, since then $1 - \sum_{j=k}^d p_{d,j} = 0$. We can therefore compute the sum in Property \ref{prop:sum} as
    \[
        \sum_{j=0}^d p_{d,j} = \sum_{j=k}^dp_{d,j} = \sum_{j=k+1}^d p_{d,j} + \left(1 - \sum_{j=k+1}^d p_{d,j}\right) = 1.
    \]
    
    Now, to choose the rest of the values $p_{i,j}$, we construct a smaller instance of the problem with values $x' := x'_0, \dots x'_{d-1}$ and $y' := y'_0, \dots y'_{d-1}$. In this instance, we let $x'_i = x_i$ and $y'_j = y_j - x_d\cdot p_{d,j}$. Remark that $x'_i = x_i \geq 0$ and $y'_j = y_j - x_d\cdot p_{d,j} \geq y_j - x_d\cdot \frac{y_j}{x_d} = 0$, so non-negativity still holds. Also, they still have equal sum since
    \begin{align*}
        \sum_{j=0}^{d-1}y'_j
        = \sum_{j=0}^{d-1} (y_j - x_d\cdot p_{d,j})
        &= -y_d + x_d\cdot p_{d,d} + \sum_{j=0}^d y_j - x_d\sum_{j=0}^d  p_{d,j} \\
        &=-y_d + x_d \cdot \frac{y_d}{x_d} + \sum_{i=0}^dx_i -x_d
        = -x_d + \sum_{i=0}^d x_i
        = \sum_{i=0}^{d-1} x'_i.
    \end{align*}
    Finally, we show that $x'$ and $y'$ satisfy the dominating property. For any $k \in \{0, \dots, d-1\}$ we will consider 2 cases. The first case is that $\sum_{j=0}^k p_{d,j} = 0$. Then, we have that
    \begin{align*}
        \sum_{i=0}^{k} x'_i &= \sum_{i=0}^{k}x_i
        \leq \sum_{j=0}^k y_j
        = \sum_{j=0}^k (y'_j + x_d\cdot p_{d,j}) = \sum_{j=0}^k y'_j.
    \end{align*}
    Now, for the other case $\sum_{j=0}^k p_{d,j} \neq 0$, we realize that this must imply that $p_{d,j} = \frac{y_j}{x_d}$ for all $j > k$, since otherwise, the sum would have been $0$. This also means that $y'_j = y_j - x_d\cdot p_{d,j} = 0$ for $j > k$. Therefore, we get that
    \begin{align*}
        \sum_{i=0}^k x'_i \leq \sum_{i=0}^{d-1} x'_i = \sum_{j=0}^{d-1} y'_j = \sum_{i=0}^k y'_j.
    \end{align*}
    We thus conclude that this smaller instance satisfies all three  conditions. The induction hypothesis therefore tells us that there exist $p'_{i,j}$ values satisfying Properties \ref{prop:non_neg}, \ref{prop:sum} and \ref{prop:eq_q} for $x',y'$. We will use these values in the problem for $x,y$. That is, we choose $p_{i,j} = p'_{i,j}$ for all $0 \leq j \leq i \leq d-1$. The $p_{d,k}$ values were already specified above. It remains to show that Properties \ref{prop:non_neg}, \ref{prop:sum} and \ref{prop:eq_q} are satisfied for $x, y,p$.

    Properties \ref{prop:non_neg} and \ref{prop:sum}  follow directly from the induction hypothesis, and the justification for the $p_{d,k}$ values given above.
    
    
    
    For Property \ref{prop:eq_q}, we can see that for any $j \in [d-1]$, by rewriting the sum
    \begin{align*}
        \sum_{i = j}^d x_i \cdot p_{i,j}
        = x_d\cdot p_{d,j} + \sum_{i=j}^{d-1} x'_i\cdot p'_{i,j}
        = x_d\cdot p_{d,j} + y'_j
        = x_d\cdot p_{d,j} + y_j - x_d\cdot p_{d,j}
        = y_j.
    \end{align*}
    And finally, for $j = d$, we also have Property \ref{prop:eq_q} directly from the definition of $p_{d,d}$.
\end{proof}

\section*{Acknowledgements}
CP was supported by Gregory Valiant's and Moses Charikar's Simons Investigator Awards, and a Google PhD Fellowship. KGL, MEM and CS are supported by the European Union (ERC, TUCLA, 101125203). Views and opinions expressed are however those of the author(s) only and do not necessarily reflect those of the European Union or the European Research Council. Neither the European Union nor the granting authority can be held responsible for them. 

\bibliographystyle{unsrtnat}
\bibliography{refs}
\end{document}